\documentclass[11pt]{article}

\usepackage{fullpage,times,url}

\usepackage{amsthm,amsfonts,amsmath,amssymb,epsfig,color,float,graphicx,verbatim}
\usepackage{algorithm,algorithmic}

\usepackage{hyperref}
\hypersetup{
	colorlinks   = true, 
	urlcolor     = blue, 
	linkcolor    = blue, 
	citecolor   = black 
}
\usepackage{subcaption}
\newtheorem{theorem}{Theorem}[section]

\newtheorem{lemma}[theorem]{Lemma}

\newtheorem{assumption}[theorem]{Assumption}
\newtheorem{remark}[theorem]{Remark}

\newcommand{\reals}{\mathbb{R}}
\newcommand{\E}{\mathbb{E}}

\newcommand{\ba}{\mathbf{a}}
\newcommand{\be}{\mathbf{e}}
\newcommand{\bx}{\mathbf{x}}
\newcommand{\bw}{\mathbf{w}}
\newcommand{\bg}{\mathbf{g}}
\newcommand{\bb}{\mathbf{b}}
\newcommand{\bu}{\mathbf{u}}
\newcommand{\bv}{\mathbf{v}}

\newcommand{\by}{\mathbf{y}}

\newcommand{\Ocal}{\mathcal{O}}

\newcommand{\Dcal}{\mathcal{D}}

\newcommand{\norm}[1]{\|#1\|}
\newcommand{\inner}[1]{\langle#1\rangle}

\newcommand{\secref}[1]{Sec.~\ref{#1}}
\newcommand{\subsecref}[1]{Subsection~\ref{#1}}

\renewcommand{\eqref}[1]{Eq.~(\ref{#1})}
\newcommand{\lemref}[1]{Lemma~\ref{#1}}

\newcommand{\thmref}[1]{Thm.~\ref{#1}}

\newcommand{\appref}[1]{Appendix~\ref{#1}}

\usepackage{bbm}
\usepackage[square,numbers]{natbib}

\title{Learning a Single Neuron with Gradient Methods}
\author{Gilad Yehudai\qquad Ohad Shamir\\
Weizmann Institute of Science\\
\texttt{\{gilad.yehudai,ohad.shamir\}@weizmann.ac.il}
}
\date{}
\begin{document}

\maketitle

\begin{abstract}
    We consider the fundamental problem of learning a single neuron $\bx\mapsto \sigma(\bw^\top\bx)$ in a realizable setting, using standard gradient methods with random initialization, and under general families of input distributions and activations. On the one hand, we show that some assumptions on both the distribution and the activation function are necessary. On the other hand, we prove positive guarantees under mild assumptions, which go significantly beyond those studied in the literature so far. We also point out and study the challenges in further strengthening and generalizing our results. 
\end{abstract}

\section{Introduction}

In recent years, much effort has been devoted to understanding why neural networks are successfully trained with simple, gradient-based methods, despite the inherent non-convexity of the learning problem. However, our understanding of this is still partial at best.

In this paper, we focus on the simplest possible nonlinear neural network, composed of a single neuron, of the form $\bx\mapsto \sigma(\bw^\top\bx)$, where $\bw$ is the parameter vector and $\sigma:\reals\rightarrow\reals$ is some fixed non-linear activation function. Moreover, we consider a realizable setting, where the inputs are sampled from some distribution $\Dcal$, the target values are generated by some unknown target neuron $\bx\mapsto \sigma(\bv^\top\bx)$ (possibly corrupted by independent zero-mean noise, and where we generally assume $\norm{\bv}=1$ for simplicity), and we wish to train our neuron with respect to the squared loss. Mathematically, this boils down to minimizing the following objective function:
\begin{equation}\label{eq:single neuron}
F(\bw) := 
\E_{\bx\sim\Dcal}\left[\frac{1}{2}\left(\sigma(\bw^\top\bx)-\sigma(\bv^\top\bx)\right)^2\right].
\end{equation}
For this problem, we are interested in the performance of gradient-based methods, which are the workhorse of modern machine learning systems. These methods initialize $\bw$ randomly, and proceed by taking (generally stochastic) gradient steps w.r.t. $F$. If we hope to explain the success of such methods on complicated neural networks, it seems reasonable to expect a satisfying explanation for their convergence on single neurons.

Although the learning of single neurons was studied in a number of papers (see the related work section below for more details), the existing analyses all suffer from one or several limitations: Either they apply for a specific distribution $\Dcal$, which is convenient to analyze but not very practical (such as a standard Gaussian distribution); Apply to gradient methods only with a specific initialization (rather than a standard random one); Require technical conditions on the input distribution which are not generally easy to verify; Or require smoothness and strict monotonicity  conditions on the activation function $\sigma(\cdot)$ (which excludes, for example, the common ReLU function $\sigma(z)=\max\{0,z\}$). However, a bit of experimentation strongly suggests that none of these restrictions is really necessary for standard gradient methods to succeed on this simple problem. Thus, our understanding of this problem is probably still incomplete. 

The goal of this paper is to study to what extent the limitations above can be removed, with the following contributions:
\begin{itemize}
\item We begin by asking whether positive results are possible without \emph{any} explicit assumptions on the distribution $\Dcal$ or the activation $\sigma(\cdot)$ (other than, say, bounded support for the former and Lipschitz continuity for the latter). Although this seems reasonable at first glance, we show in \secref{sec:assumptions} that unfortunately, this is not the case: Even for the ReLU activation function, there are bounded distributions $\Dcal$ on which gradient descent will fail to optimize \eqref{eq:single neuron} with probability exponentially close to $1$. Moreover, even for $\Dcal$ which is a standard Gaussian, there are Lipschitz activation functions on which gradient methods will likely fail. 
\item Motivated by the above, we ask whether it is possible to prove positive results with \emph{mild and transparent} assumptions on the distribution and activation function, which does not exclude common setups. In \secref{sec:mild assumptions gradient}, we prove a key technical result, which implies that if the distribution $\Dcal$ is sufficiently ``spread'' and the activation function satisfies a weak monotonicity condition (satisfied by ReLU and all standard activation functions), then $\inner{\nabla F(\bw),\bw-\bv}$ is positive in most of the domain. This implies that an exact gradient step with sufficiently small step size will bring us closer to $\bv$ in ``most'' places. Building on this result, we prove in \secref{sec:convergence with constant probability} a constant-probability convergence guarantee for several variants of gradient methods (gradient descent, stochastic gradient descent, and gradient flow) with random initialization.  
\item In \secref{sec:spherically symmetric}, we consider more specifically the case where $\Dcal$ is any spherically symmetric distribution (which includes the standard Gaussian as a special case) and the ReLU activation function. In this setting, we show that the convergence results can be made to hold with high probability, due to the fact that the angle between the parameter vector and the target vector $\bv$ motonically decreases. As we discuss later on, the case of the ReLU function and a standard Gaussian distribution was also considered in  \cite{soltanolkotabi2017learning,kalan2019fitting}, but that analysis crucially relied on initialization at the origin and a Gaussian distribution, whereas our results apply to more generic initialization schemes and distributions.
\item A natural question arising from these results is whether a high-probability result can be proved for non-spherically symmetric distributions. We study this empirically in Subsection \ref{subsec:failures of gaussian}, and show that perhaps surprisingly, the angle to the target function might  \emph{increase} rather than decrease, already when we consider unit-variance Gaussian distributions with a non-zero mean. This suggests that a fundamentally different approach would be required for a general high-probability guarantee. 
\end{itemize}

Overall, we hope our work contributes to a better understanding of the dynamics of gradient methods on simple neural networks, and suggests some natural avenues for future research.

\subsection{Related Work}

First, we emphasize that learning a single target neuron is \emph{not} an inherently difficult problem: Indeed, it can be efficiently performed with minimal assumptions, using the Isotron algorithm and its variants (\citet{kalai2009isotron,kakade2011efficient}). Also, other algorithms exist for even more complicated networks or more general settings, under certain assumptions (e.g., \citet{goel2016reliably,janzamin2015beating}). However, these are non-standard algorithms, whereas our focus here is on standard, vanilla gradient methods. 

For this setting, a positive result was provided in \citet{mei2016landscape}, showing that gradient descent on the empirical risk function $\frac{1}{n}\sum_{i=1}^{n}(\sigma(\bx_i^\top\bw)-\sigma(\bx_i^\top\bv))^2$ (with $\bx_i$ sampled i.i.d. from $\Dcal$ and $n$ sufficiently large) successfully yields a good approximation of $\bv$. However, the analysis requires $\sigma$ to be strictly monotonic, and to have uniformly bounded derivatives up to the third order. This excludes standard activation functions such as the ReLU, which are neither strictly monotonic nor differentiable. Indeed, assuming that the activation is strictly monotonic makes the analysis much easier, as we show later on in \thmref{thm:too strong activation assumption}. A related analysis under strict monotonicity conditions is provided in \citet{oymak2018overparameterized}. 

For the specific case of a ReLU activation function $\sigma(\cdot)=\max\{\cdot,0\}$ and a standard Gaussian input distribution,  \citet{tian2017analytical} proved that with constant probability, gradient flow over \eqref{eq:single neuron} will asymptotically converge to the global minimum. \citet{soltanolkotabi2017learning} and \citet{kalan2019fitting} considered a similar setting, and proved a non-asymptotic convergence guarantee for gradient descent or stochastic gradient descent on the empirical risk function $\frac{1}{n}\sum_{i=1}^{n}(\sigma(\bx_i^\top\bw)-\sigma(\bx_i^\top\bv))^2$. However, that analysis crucially relied on initialization at precisely $\mathbf{0}$, as well as a certain assumption on how the derivative of the ReLU function is computed at $0$. In more details, we impose the convention that even though the ReLU function is not differentiable at $0$, we take $\sigma'(0)$ to be some fixed positive number, and the gradient of the population objective $F$ at $\mathbf{0}$ to be
\[
\E_{\bx\sim \Dcal}\left[(\sigma(0)-\sigma(\bv^\top\bx))\sigma'(0)\bx\right]~=~
-\sigma'(0)\cdot\E_{\bx\sim \Dcal}\left[\sigma(\bv^\top\bx)\bx\right]~.
\]
Assuming $\sigma'(0)>0$, we get that the gradient is non-zero and proportional to $-\E_{\bx\sim\Dcal}[\sigma(\bv^\top\bx)\bx]$. For a Gaussian distribution (and more generally, spherically symmetric distributions), this turns out to be proportional to $-\bv$, so that an exact gradient step from $\mathbf{0}$ will lead us precisely in the direction of the target parameter vector $\bv$. As a result, if we calculate a sufficiently precise approximation of this direction from a random sample, we can get \emph{arbitrarily close to $\bv$ in a single iteration} (see  \citet[Remark 1]{kalan2019fitting} for a discussion of this). Unfortunately, this unique behavior is specific to initialization at $\mathbf{0}$ with a certain convention about $\sigma'(0)$ (note that even locally around $\mathbf{0}$, the gradient may not approximate $\bv$, since it is generally discontinuous around $\mathbf{0}$). Thus, although the analysis is important and insightful, it is difficult to apply more generally. 

\citet{du2017convolutional} considered conditions under which a single ReLU convolutional filter is learnable with gradient methods, a special case of which is a single ReLU neuron. The paper is closely related to our work, in the sense that they were also motivated by finding general conditions under which positive results are attainable. Moreover, some of the techniques they employed share similarities with ours (e.g., considering the gradient correlation as in \secref{sec:mild assumptions gradient}). However, our results differ in several aspects: First, they consider only the ReLU activation function, while we also consider general activations. Second, their results assume a technical condition on the eigenvalues of certain distribution-dependent matrices, with the convergence rate depending on these eigenvalues. However, the question of when might this condition hold (for general distributions) is left unclear. In contrast, our assumptions are more transparent and have a clear geometric intuition. Third, their results hold with constant probability, even for a standard Gaussian distribution, while we employ a different analysis to prove high probability guarantees for general spherically symmetric distributions. Finally, we also provide negative results, showing the necessity of assumptions on both the activation function and the input distribution, as well as suggesting which approaches might not work for further generalizing our results. 


A line of recent works established the effectiveness of gradient methods in solving non-convex optimization problems with a \emph{strict saddle} property, which implies that all near-stationary points with nearly positive definite Hessians are close to global minima (see \citet{jin2017escape,ge2015escaping,sun2015nonconvex}). A relevant example is phase retrieval, which actually fits our setting with $\sigma(\cdot)$ being the quadratic function $z\mapsto z^2$ (\citet{sun2018geometric}). However, these results can only be applied to smooth problems, where the objective function is twice differentiable with Lipschitz-continuous Hessians (excluding, for example, problems involving the ReLU activation function). An interesting recent exception is the work of \citet{tan2019online}, which considered the case $\sigma(z)=|z|$. However, their results are specific to that activation, and assumes a specific input distribution $\Dcal$ (uniform on a scaled origin-centered sphere). In contrast, our focus here is on more general families of distributions and activations. 

\citet{brutzkus2017globally} show that gradient descent learns a simple convolutional network with non-overlapping patches, when the inputs have a standard Gaussian distribution. Similar to the analysis in \secref{sec:spherically symmetric} in our paper, they rely on showing that the angle between the learned parameter vector and a target vector monotonically decreases with gradient methods. However, the network architecture studied is different than ours, and their proof heavily relies on the symmetry of the Gaussian distribution.

Less directly related to our setting, a popular line of recent works showed how gradient methods on highly over-parameterized neural networks can learn various target functions in polynomial time (e.g., \citet{allen2019learning,daniely2017sgd,arora2019fine,cao2019generalization}). However, as pointed out in \citet{yehudai2019power}, this type of analysis cannot be used to explain learnability of single neurons.

\section{Preliminaries}\label{sec:preliminaries}

\textbf{Notation.} We use bold-faced letters to denote vectors. For a vector $\bw$, we let $w_i$ denote its $i$-th coordinate. We denote $[z]_+ := \max\{0,z\}$ to be the ReLU function. For a vector $\bw$, we let $\bar{\bw} := \frac{\bw}{\norm{\bw}}$, and by $\pmb{1}$ we denote the all-ones vector $(1,\dots,1)$. Given vectors $\bw$, $\bv$ we let $\theta(\bw,\bv) := \arccos\left(\frac{\bw^\top\bv}{\norm{\bw}\norm{\bv}}\right)=\arccos(\bar{\bw}^\top\bar{\bv})\in [0,\pi]$ denote the angle between $\bw$ and $\bv$. We use $\mathcal{P}$ to denote probability. $\mathbbm{1}(\cdot)$ denotes the indicator function, for example $\mathbbm{1}(x>0)$ equals $1$ if $x>0$ and $0$ otherwise.

\textbf{Target Neuron.} Unless stated otherwise, we assume that the target vector $\bv$ in \eqref{eq:single neuron} is unit norm, $\norm{\bv}=1$.

\textbf{Gradients.} When $\sigma(\cdot)$ is differentiable, the gradient of the objective function in \eqref{eq:single neuron} is 
\begin{equation}\label{eq:gradient of a single neuron}
\nabla F(\bw) = 
\E_{\bx\sim\Dcal}\left[\left(\sigma(\bw^\top\bx)-\sigma(\bv^\top\bx)\right)
\cdot\sigma'(\bw^\top\bx)\bx\right]    
\end{equation}
When $\sigma(\cdot)$ is not differentiable,
we will still assume that it is differentiable almost everywhere (up to a finite number of points), and that in every point of non-differentiability $z$, there are well-defined left and right derivatives. In that case, practical implementations of gradient methods fix $\sigma'(z)$ to be some number between its left and right derivatives (for example, for the ReLU function, $\sigma'(0)$ is defined as some number in $[0,1]$). Following that convention, the expected gradient used by these methods still corresponds to \eqref{eq:gradient of a single neuron}, and we will follow the same convention here. 

\textbf{Algorithms.} In our paper, we focus on the following three standard gradient methods:
\begin{itemize}
    \item \textbf{Gradient Descent}: We 
initialize at some $\bw_0$ and set a fixed learning rate $\eta$. At each iteration $t>0$, we do a single step in the negative direction of the gradient: $\bw_{t+1} = \bw_t - \eta\nabla F(\bw_t).$
    \item \textbf{Stochastic Gradient Descent (SGD)}:
We initialize at some $\bw_0$ and set a fixed learning rate $\eta$. At each iteration $t>0$, we sample an input $\bx_t \sim \mathcal{D}$, and calculate a stochastic gradient:
\begin{equation}\label{eq:stochastic gradient}
g_t = \left(\sigma(\bw_t^\top\bx_t)-\sigma(\bv^\top\bx_t)\right)
\cdot\sigma'(\bw_t^\top\bx_t)\bx_t    
\end{equation}
and do a single step in the negative direction of the stochastic gradient: $\bw_{t+1} = \bw_t - \eta g_t.$
Note that here we consider SGD on the population loss, which is different from SGD on a fixed training set. We also note that our  proof techniques easily extend to mini-batch SGD, where $g_t$ is taken to be the average of $B$ stochastic gradients w.r.t. $\bx_t^1,\ldots,\bx_t^B$ sampled i.i.d. from $\Dcal$. However, for simplicity we will focus on $B=1$.
    \item \textbf{Gradient Flow}: We initialize at some $\bw(0)$, and for every $t>0$, we set $\bw(t)$ to be the solution of the differential equation:
$\Dot{\bw}(t) = -\nabla F(\bw (t)).$
This can be thought of as a continuous form of gradient descent, where we consider an infinitesimal learning rate. We note that strictly speaking, gradient flow is not an algorithm. However, it approximates the behavior of gradient descent in many cases, and has the advantage that its analysis is often simpler. 
\end{itemize}


\section{Assumptions on the Distribution and Activation are Necessary}\label{sec:assumptions}

The main concern of this paper is under what assumptions can a single neuron be provably learned with gradient methods. In this section, we show that perhaps surprisingly, this is not possible unless we make non-trivial assumptions on \emph{both} the input distribution and the activation function. 

\subsection{Assumptions on the Input Distribution are Necessary}

We begin by asking whether \eqref{eq:single neuron} can be minimized by gradient methods in a distribution-free manner (with no assumptions beyond, say, bounded support), as in learning problems where the population objective is convex. Perhaps surprisingly, we show that the answer is negative, even if we consider specifically the ReLU activation, and a distribution supported on the unit Euclidean ball. This is based on the following key result:

\begin{theorem}\label{thm:cannot learn unit vectors}
Suppose that $\sigma$ is the ReLU function (with the convention that $\sigma'(z)=\mathbbm{1}(z>0)$), and assume that $\bw$ is sampled from a product distribution $D_{\bw}$ (namely, each $w_i$ is sampled independently from some distribution ${D}_{\bw}^i$). Then there exists a distribution $\mathcal{D}$ over the inputs, supported on $\{\bx :\norm{\bx}\leq 1\}$, and $\bv$ with $\norm{\bv}=1$ such that the following holds: With probability at least $1 - \exp\left(-\frac{d}{4}\right)$ over the initialization point sampled from $D_{\bw}$, if we run gradient flow, gradient descent or stochastic gradient descent, then for every $t>0$  we have $ F(\bw_t)-\inf_{\bw}F(\bw) \geq \frac{1}{8d}$ (for gradient flow  $ F(\bw(t))-\inf_{\bw}F(\bw) \geq \frac{1}{8d}$).
\end{theorem}

\begin{proof}
For each distribution $\mathcal{D}_{\bw}^i$, let $p_i = \mathcal{P}(w_i > 0)$. We define the following dataset: \[
    S = \{\bx_i = b_i\be_i:i=1\dots,d  \}
    \]
    where $\be_i$ is the standard $i$-th unit vector, and $b_i=1$ if $p_i < \frac{1}{2}$ and $-1$ otherwise. Take $\mathcal{D}$ to be the uniform distribution on $S$. 
    
    Informally, the proof idea is the following: With overwhelming probability, we will initialize at a point $\bw$ such that for at least $\Omega(d)$ coordinates $i$, it holds that $\sigma'(\bw^\top\bx_i)=0$, and as a result, $\nabla F(\bw)$ is zero on those coordinates. Based on this, we show that these coordinates will not change from their initialized values. However, a point $\bw$ with $\Omega(d)$ coordinates with this property is suboptimal by a fixed factor, so the algorithm does not converge to an optimal solution.
    
    More formally, using \eqref{eq:gradient of a single neuron} and the fact that $\sigma$ is the ReLU function, we get
    \[
    \nabla F(\bw) = \frac{1}{d}\sum_{i=1}^d \left(\sigma(\bw^\top\bx_i)-\sigma(\bv^\top\bx_i)\right)
    \cdot\mathbbm{1}\left(\bw^\top\bx_i > 0\right)\bx_i~.
    \]
    In particular, for every index $i$ for which $\mathbbm{1}\left(\bw^\top \bx_i > 0\right)=0$ we have that $\left(\nabla F(\bw)\right)_i = 0$. Next, we define $\bv$ with $\bv_i = b_i\frac{1}{\sqrt{d}}$ (note that $\norm{\bv} = 1$). For every $d/4$ indices $i_1,\dots, i_{d/4}$ for which $\mathbbm{1}\left(\bw^\top \bx_i \geq 0\right)=0$ we have that:
    \begin{align}\label{eq:lower bound on objective}
        F(\bw) &= \frac{1}{2d}\sum_{i=1}^d\left(\sigma(\bw^\top\bx_i) - \sigma(\bv^\top\bx_i)\right)^2 \geq \frac{1}{2d} \sum_{i\in \{i_1,\dots,i_{d/4}\}}\left(\sigma(\bw^\top\bx_i) - \sigma(\bv^\top\bx_i)\right)^2 \nonumber\\
        = & \frac{1}{2d}\sum_{i\in \{i_1,\dots,i_{d/4}\}} \sigma(\bv^\top \bx_i)^2 = \frac{1}{2d}\sum_{i\in \{i_1,\dots,i_{d/4}\}} \sigma\left(b_i^2 \frac{1}{\sqrt{d}}\right)^2 = \frac{1}{8d}
    \end{align}
    
Denote the random variable $Z_i = \mathbbm{1}\left(\bw_0^\top \bx_i > 0\right)$ and $Z = \sum_{i=1}^d Z_i$ (for gradient flow we denote $Z_i = \mathbbm{1}\left(\bw(0)^\top \bx_i \geq 0\right)$). It is easily verified that $\E[Z_i]=\Pr(\bw_0^\top\bx_i>0)=\Pr(w_{0,i}b_i>0)\leq \frac{1}{2}$. We have that $Z_1,\ldots,Z_d$ are independent, $\max_i |Z_i| \leq 1$, and $\mathbb{E}[Z] =\sum_{i=1}^{d}\E[Z_i]\leq \frac{d}{2}$. Using Hoeffding's inequality, we get that w.p $\geq 1-\exp\left(-\frac{d}{4}\right)$ it holds that $Z \leq \frac{3}{4}d$, which means that there are at least $\frac{d}{4}$ indices such that $Z_i =0$. We condition on this event and let these indices be $i_1,\dots,i_{d/4}$.
We will now show that for every index $i\in\{i_1,\dots,i_{d/4}\}$, using gradient methods will not change the $i$-th coordinate of $\bw_t$ ($\bw(t)$ for gradient flow) from its initial value. Let $i$ be such a coordinate. 

For gradient descent, we will show by induction that for every iteration $t$ we have that $\mathbbm{1}\left(\bw_t^\top \bx_i > 0\right)=0$. The base case is true, because we conditioned on this event. Assume for $t-1$, then $\left(\nabla F(\bw_{t-1})\right)_i =0$, which means that $(\bw_{t})_i = (\bw_{t-1})_i - \eta (\nabla F(\bw_{t-1}))_i = (\bw_{t-1})_i$, and in particular $\mathbbm{1}\left(\bw_t^\top \bx_i > 0\right)=\mathbbm{1}\left(\bw_{t-1}^\top \bx_i > 0\right)=0$. This proves that for every iteration $t$, the $i$-th coordinate of $\nabla F(\bw_t)$ is zero, which mean that $(\bw_t)_i = (\bw_0)_i$. 

For stochastic gradient descent, at each iteration $t$ we sample $\bx_t\sim \mathcal{D}$, and define the stochastic gradient $g_t$ as in \eqref{eq:stochastic gradient}. If $\bx_t \neq \bx_i$ then $(\bx_t)_i =0$ hence $(g_t)_i = 0$, otherwise, if $\bx_t=\bx_i$ then by $(g_t)_i =(\nabla F(\bw_{t}))_i$ and by the same induction argument as in gradient descent we have that $(g_t)_i = 0$. In both cases the $i$-th coordinate of the stochastic gradient is zero, hence $(\bw_t)_i = (\bw_0)_i$.


For gradient flow, assume on the way of contradiction that for some $t>0$ that $\mathbbm{1}\left(\bw(t)^\top \bx_i > 0\right) \neq 0$ and let $t_1$ be the first time that this happen. Then for all $0<t<t_1$ we have that $\mathbbm{1}\left(\bw(t)^\top \bx_i > 0\right) = 0$, and in particular $\left(\nabla F(\bw(t))\right)_i =0$. Hence for all $0<t<t_1$ running gradient flow we get $\left(\Dot{\bw}(t)\right)_i = \left(\nabla F(\bw(t))\right)_i = 0$, and in particular $\mathbbm{1}\left(\bw(t)^\top \bx_i > 0\right) = \mathbbm{1}\left(\bw(0)^\top \bx_i > 0\right) = 0$, a contradiction to the fact that $\bw(t)$ is continuous. Thus for all $t>0$ we showed that $\mathbbm{1}\left(\bw(t)^\top \bx_i > 0\right) = 0$, hence $\left(\nabla F(\bw(t))\right)_i =0$ which shows that $(\bw(t))_i = (\bw(0))_i$. 

By the conditioned event, \eqref{eq:lower bound on objective} applies at initialization. Since in all the gradient methods above the $i$-th coordinate of $\bw$ did not change from its initial value for $i\in\{i_1,\dots,i_{d/4}\}$, we can apply \eqref{eq:lower bound on objective} to get that for every iteration $t >0$ for gradient descent or SGD we have that $F(\bw_t) \geq \frac{1}{8d}$ (and for gradient flow, for every time $t>0$, we have $F(\bw(t)) \geq \frac{1}{8d}$).

We end by noting that although the distribution defined here is discrete over a finite dataset, the same argument can also be made for a non-discrete distribution, by considering a mixture of smooth distributions concentrated around the support points of the discrete distribution above.
\end{proof}

The theorem above applies to any product initialization scheme, which includes most standard initializations used in practice (e.g., the standard Xavier initialization \cite{glorot2010understanding}). The theorem implies that
it is impossible to prove positive guarantees in our setting without distributional assumptions on ths inputs. Inspecting the construction, the source of the problem (at least for the ReLU neuron) appears to be the fact that the distribution is supported on a small number of well-separated regions. Thus, in our positive results, we will assume that the distribution is sufficiently ``spread'', as formalized later on in \secref{sec:mild assumptions gradient}

\subsection{Assumptions on the Activation Function}

We now turn to discuss the activation function, explaining why even if the activation is Lipschitz and the input distribution $\Dcal$ is a standard Gaussian, this is likely insufficient for positive guarantees in our setting.

In particular, let us consider the case that $\sigma(\cdot)$ is a $1$-Lipschitz \emph{periodic} function. Then Theorem $3$ in \cite{shamir2018distribution} implies that for a large family of input distributions $\Dcal$ on $\reals^d$ (including a standard Gaussian), if we assume that the vector $\bv$ in the target neuron $\sigma(\bv^\top\bx)$ is a uniformly distributed unit vector, then for any fixed $\bw$,
\begin{equation*}
Var_{\bv}(\nabla F(\bw)) \leq \Ocal(\exp(-d)).    
\end{equation*}
This implies that the gradient at $\bw$ is virtually independent of the underlying target vector $\bv$: In fact, it is extremely concentrated around a fixed value which does not depend on $\bv$. Theorem 4 from  \cite{shamir2018distribution} goes further and shows that for any gradient method, even an exponentially small amount of noise will be enough to make its trajectory (after at most $\exp(\Ocal(d))$ iterations) independent of $\bv$, in which case it cannot possibly succeed in this setting.
We note that their result is even more general as they consider a general function $f(\bw,\bx)$ instead of $\sigma(\inner{\bw,\bx})$, so our setting can be seen as a private case. 

When considering a standard Gaussian distribution, the above argument can be easily extended to activations $\sigma$ which are periodic only in a segment of length $\Omega(d)$ around the origin. This can be seen by extending the activation to $\tilde{\sigma}$ which is periodic on $\mathbb{R}$, applying the above argument to it, and noting that the probability mass outside of a ball of radius $\Omega(d)$ is exponentially small (for example, see \cite{yehudai2019power} Proposition 4.2, where they consider an activation which is a finite sum of ReLU functions and periodic in a segment of length $O(d^2)$). 

The above discussion motivates us to impose some condition on the activation function which excludes periodic functions. One such mild assumptions, which we will adopt in the rest of the paper (and corresponds to virtually all activations used in practice) is that the activation is monotonically non-decreasing. Before continuing, we remark that by assuming a slight strengthening of this assumption, namely that the function is \emph{strictly} monotonically increasing, it is easy to prove a positive guarantee, as evidenced by \thmref{thm:too strong activation assumption}. However, this excludes popular activations such as the ReLU function. 

\begin{theorem}\label{thm:too strong activation assumption}
Assume $\inf_z \sigma'(z) \geq \gamma >0$ for some $\gamma >0$, and the following for some $\lambda,c_1,c_2$:
\begin{itemize}
    \item $\Sigma := \mathbb{E}_\bx\left[\bx\bx^\top\right]$ is positive definite with minimal eigenvalue $\lambda > 0$
    \item $\mathbb{E}_{\bx\sim \mathcal{D}}\left[\norm{\bx}^2\right] \leq c_1$
    \item $\sup_{z}\sigma'(z) \leq c_2$~.
\end{itemize}
Then starting from any point $\bw_0$, 
after doing $t$ iterations of gradient descent with learning rate $\eta < \frac{\lambda\gamma^2}{c_1^2c_2^4}$, we have that:
\[
\norm{\bw_t - \bv}^2 \leq \norm{\bw_0 - \bv}(1-\lambda\gamma^2\eta)^t~.
\]

\end{theorem}

The proof can be found in \appref{appen:proof from sec assumptions}, and can be easily generalized to apply also to gradient flow and SGD. The above shows that if we assume strict monotonicity of the activation, then under very mild assumptions on the data $\bw_t$ will converge exponentially fast to $\bv$. In the rest of the paper, however, we focus on results which only require weak monotonicity.

\section{Under Mild Assumptions, the Gradient Points in a Good Direction}\label{sec:mild assumptions gradient}

Motivated by the results in \secref{sec:assumptions}, we use the following assumptions on the distribution and activation:

\begin{assumption}\label{assump:D and sigma}
    The following holds for some fixed $\alpha,\beta,\gamma > 0$:
    \begin{enumerate}
        \item The distribution $\Dcal$ satisfies the following: For any vector $\bw\neq \bv$, let $\Dcal_{\bw,\bv}$ 
	denote the marginal 
	distribution of $\bx$ on the subspace spanned by $\bw,\bv$ (as a 
	distribution 
	over $\reals^2$). Then any such distribution has a density function 
	$p_{\bw,\bv}(\bx)$ such that $\inf_{\bx:\norm{\bx}\leq 
		\alpha}p_{\bw,\bv}(\bx)\geq \beta$. 
		\item $\sigma:\reals\mapsto\reals$ is monotonically non-decreasing, and satisfies 
	$\inf_{0 < z < 2\alpha}\sigma'(z)\geq \gamma$. 
    \end{enumerate}
\end{assumption}

The distributional assumption is such that in every $2$-dimensional subspace,
the marginal distribution is sufficiently ``spread'' in any direction close to the origin. For example, for a standard Gaussian distribution, this is true for $\alpha,\beta=\Theta(1)$ 
regardless of the dimension $d$ (as the marginal distribution of a standard Gaussian on the subspace is a standard $2$-dimensional Gaussian). Also, for any distribution, it can be made to hold by mixing it with a bit of a Gaussian or uniform distribution if possible. 
The assumption on the activation function is very mild, and covers most activations used in practice such as ReLU and ReLU-like functions (e.g. leaky-ReLU, Softplus), as well as standard sigmoidal activations (for which the derivative in any bounded interval is lower bounded by a positive constant).

With these assumptions, we prove the following key technical result, which implies that the gradient of the objective has a positive correlation with the direction of the global minimum (at $\bw=\bv$), if the angle between $\bw$ and $\bv$ and the norm of $\bw$ are not too large:
 \begin{theorem}\label{thm:innerprod}
	Under Assumptions \ref{assump:D and sigma}, for any $\bw$ such that $\norm{\bw} \leq 2$ and  $\theta(\bw,\bv)\leq \pi-\delta$ for some $\delta\in (0,\pi]$, it holds 
	that 
	\[
	\inner{\nabla F(\bw),\bw-\bv}~\geq~
	\frac{\alpha^4 \beta \gamma^2}{8\sqrt{2}}\sin^3\left(\frac{\delta}{4}\right)\norm{\bw-\bv}^2~.
	\]
\end{theorem}
The theorem implies that for suitable values of $\bw$, gradient methods (which move in the negative gradient direction) will decrease the distance from $\bv$. When this behavior occurs, it is easy to show that gradient methods succeed in learning the target neuron, like in the previous \thmref{thm:too strong activation assumption} for the strictly monotonic case. The main challenge is to guarantee that the trajectory of the algorithm will indeed never violate the theorem's conditions, in particular that the angle between $\bw$ and $\bv$ indeed remains bounded away from $\pi$ (and in fact, later on we will show that such a guarantee is not always possible).

The formal proof of the theorem can be found in \appref{appen:proofs from sec mild}, but its intuition can be described as follows: we want to bound below the term
\begin{equation*}
\inner{\nabla F(\bw),\bw-\bv} = 
\E_{\bx}\left[\left(\sigma(\bw^\top\bx)-\sigma(\bv^\top\bx)\right)
\cdot\sigma'(\bw^\top\bx)\cdot(\bw^\top\bx-\bv^\top\bx)\right]~.
\end{equation*}
Note that:
\begin{enumerate}
    \item Using the assumption on $\sigma$, the term inside the above expectation is nonnegative for every $\bx$. This is because $\sigma'(x) \geq 0$, and for any monotonically non-decreasing function $f$ we have $(f(x)-f(y))(x-y) \geq 0$. Thus, viewing the expectation as an integral over a nonnegative function, we can lower bound it by taking the integral over the smaller set $\left\{\bx\in\mathbb{R}^d:~ \bw^\top \bx >0,~ \bv^\top \bx >0 \right\}$. Note that on this set, $\sigma(\bw^\top\bx)=\bw^\top\bx$ and $\sigma(\bv^\top\bx)=\bv^\top\bx$.
    \item The resulting integral depends only on dot products of $\bx$ with $\bw$ and $\bv$. Thus, it is enough to consider the marginal distribution on the $2$-dimensional plane spanned by $\bw$ and $\bv$.
    \item By the assumption on the distribution, the density function of this marginal distribution is always at least $\beta$ on any $\bx$ such that $\norm{\bx}\leq \alpha$. This means we can lower bound the integral above by integrating over $\bw$ with a uniform distribution on this set and multiplying by $\beta$.
\end{enumerate}

In total, the expression above can be lower bounded by a certain $2$-dimensional integral (with uniform measure and with no $\sigma$ terms) on the set
\[
\left\{\by\in\mathbb{R}^2:~ \hat{\bw}^\top \by >0,~ \hat{\bv}^\top \by >0, \norm{\by} \leq \alpha \right\}
\]
where $\hat{\bw},\hat{\bv}$ are the $2$-dimensional vectors representing $\bw,\bv$ on the $2$-dimensional plane spanned by them. We lower bound this integral by a term that scales with the angle $\theta(\bw,\bv)$.

\begin{remark}[Implication on Optimization Landscape]\label{remark:implication on landscape}
The proof of the theorem can be shown to imply that for the ReLU activation, under the theorem's conditions, the only stationary point that is not the global minimum $\bv$ must be at the origin. In particular, the proof implies that any stationary point (with $\nabla F(\bw)=0$) must be along the ray $\{\bw=-a\cdot\bv:a\geq 0\}$. For the ReLU activation (which satisfies $\sigma(z)\sigma'(-a\cdot z)=0$ for any $a\geq 0$ and $z$), the gradient at such points equals
\[
\nabla F(-a\cdot \bv) = \E_{\bx}\left[
(\sigma(-a\bv^\top\bx)-\sigma(\bv^\top\bx))
\sigma'(-a\bv^\top\bx)\bx\right]
~=~ \E_{\bx}\left[(-a\bv^\top\bx)\sigma'(-a\bv^\top\bx)\bx\right]~.
\]
In particular,
\[
\inner{\nabla F(-a\cdot \bv),\bv} = -a\cdot \E_{\bx}\left[\sigma'(-a\bv^\top\bx)
(\bv^\top\bx)^2\right]~.
\]
This implies that $\nabla F(-a\cdot\bv)$ might be zero only if either $a=0$ (i.e., at the origin), or $\bv^\top\bx\geq 0$ with probability $1$, which cannot happen according to Assumption \ref{assump:D and sigma}.
\end{remark}


\section{Convergence with Constant Probability Under Mild Assumptions}
\label{sec:convergence with constant probability}
In this section, we use \thmref{thm:innerprod} in order to show that under some assumption on the initialization of $\bw$, gradient methods will be able to learn a single neuron with probability at least (close to) $\frac{1}{2}$. Note that the loss surface of $F(\bw)$ is not convex, and as explained in Remark \ref{remark:implication on landscape}, there may be a stationary point at $\bw =\mathbf{0}$. This stationary point can cause difficulties, as it is not obvious how to control the angle between $\bv$ and $\bw$ close to the origin (which is required for \thmref{thm:innerprod} to apply). But, if we assume $\norm{\bw-\bv}^2 < 1$ at initialization, then we are bounded away from the origin, and we can ensure that it will remain that way throughout the optimization process. One such initialization, which guarantees this with at least constant probability, is a zero-mean Gaussian initialization with small enough variance:

\begin{lemma}\label{lemma:constant probability initialization}
    Assume $\|\bv\|=1$. If we sample $\bw\sim \mathcal{N}\left(0,\tau^2 I\right)$ for $\tau \leq \frac{1}{d\sqrt{2}}$ then w.p $> \frac{1}{2} - \frac{1}{4}\tau d - 1.2^{-d}$  we have that $\norm{\bw -\bv}^2 \leq 1-2\tau^2d$
\end{lemma}

In order to bound each gradient step we will need these additional assumptions:
\begin{assumption}\label{assump:gd}
    The following holds for some positive $c_1,c_2$:
    \begin{enumerate}
    \item $\norm{\bx}^2 \leq c_1$ almost surely over $\bx\sim \mathcal{D}$
    \item $\sigma'(z)\leq c_2$ for all $z\in\mathbb{R}$
\end{enumerate}
\end{assumption}

With these assumptions, we show convergence for gradient flow, gradient descent and stochastic gradient descent:
\begin{theorem}\label{thm:constant probability convergence}
Under assumptions \ref{assump:D and sigma} and \ref{assump:gd} we have:
\begin{enumerate}
    \item (Gradient Flow) Assume that $\|\bw(0) - \bv\|^2 < 1$. Running gradient flow, then for every time $t >0$ we have 
\[
\|\bw(t) - \bv\|^2 \leq \|\bw(0)-\bv\|^2\exp(-t\lambda)
\]
where $\lambda = \frac{\alpha^4\beta\gamma^2}{210}$.
    \item (Gradient Descent) Assume that $\|\bw_0 - \bv\|^2 < 1$. Let $\eta \leq \frac{\lambda}{2c}~$ for $\lambda = \min\left\{1,\frac{\alpha^4\beta\gamma^2}{210}\right\}$ and $c=c_1^2c_2^4$. Running gradient descent with step size $\eta$, we have that for every $T>0$,  after $T$ iterations:
    \[
    \norm{\bw_T-\bv}^2 \leq \norm{\bw_0-\bv}^2\left(1-\frac{\eta\lambda}{2}\right)^T
    \]
    \item (Stochastic Gradient Descent) Let $\epsilon_1,\epsilon_2,\delta >0$, and assume that $\|\bw_0 - \bv\|^2 \leq 1 - \epsilon_1$. Let $\eta \leq \frac{\lambda\epsilon_1^2\epsilon_2^2 c_3^2}{60c_1^3c_2^6 \log\left(\frac{2}{\delta}\right)}$ where $\lambda = \frac{\alpha^4\beta \gamma^2}{210}$ and $c_3 = \left( \frac{1}{2}\right)^\frac{\lambda}{20c_1c_2^2} - \left( \frac{1}{2}\right)^\frac{\lambda}{18c_1 c_2^2}$. Then w.p $1-\left\lceil\frac{20c_1c_2^2\log\left(\frac{1}{\epsilon_2}\right)}{\lambda}\right\rceil\delta$, after $T \geq  \frac{2\log\left(\frac{1}{\epsilon_2}\right)}{\lambda\eta}$ iterations we have that:
\[
\norm{\bw_T -\bv}^2 \leq \epsilon_2
\]
\end{enumerate}
\end{theorem}

Combined with \lemref{lemma:constant probability initialization}, \thmref{thm:constant probability convergence} shows that with proper initialization, gradient flow, gradient descent as well as stochastic gradient descent successfully minimize \eqref{eq:single neuron} with probability (close to) $\frac{1}{2}$, and for the first two algorithms, the distance to $\bv$ decays exponentially fast. 

The full proof of the theorem can be found in \appref{appen:proofs from constant probability}, and its intuition for gradient flow and gradient is as described above (namely, that if $\norm{\bw-\bv}<1$, it will stay that way and $\norm{\bw-\bv}$ will just continue to shrink over time, using \thmref{thm:innerprod}). The proof for stochastic gradient descent is much more delicate. This is because  the update at each iteration is noisy, so we need to ensure we remain in the region where \thmref{thm:innerprod} is applicable. Here we give a short proof intuition:
\begin{enumerate}
    \item Assume we initialized with  $\norm{\bw_0 -\bv}^2 \leq 1-\epsilon$ for some $\epsilon>0$. In order for the analysis to work we need that $\norm{\bw_t-\bv} <1$ throughout the algorithm's run. Thus, we show (using a maximal version of Azuma's inequality) that if $\eta$ is small enough (depending on $\epsilon$), and we take at most $m= O\left(\frac{1}{\eta}\right)$ gradient steps then w.h.p for every $t=1,\dots,m$: $\norm{\bw_t -\bv}^2 \leq 1- \frac{\epsilon}{2}$
    
    \item The next step is to show that if $\norm{\bw_t -\bv}^2 < 1$, then $\mathbb{E}\left[\|\bw_{t+1} - \bv\|^2|\bw_t\right] \leq (1-\eta\lambda)\|\bw_{t} - \bv\|^2$
    for an appropriate $\lambda$. This is done using \thmref{thm:innerprod}, as in the gradient descent case, but note that here this only holds in expectation over the sample selected at iteration $t$.
    
    \item Next, we use Azuma's inequality again on $m=O\left(1/\eta\right)$ iterations for a small enough $\eta$, to show that  w.h.p $\bw_m$ does not move too far away from $\tilde{\bw}_m := \mathbb{E}[\bw_m]$ where the expectation is taken over $\bx_1,\dots,\bx_m$. Also, we show that after $m$ iterations $\norm{\tilde{\bw}_m-\bv}^2 \leq \rho\norm{\bw_0 - \bv}^2$ for a constant $\rho$ smaller than $1$. This shows that w.h.p., after a single epoch of $m$ iterations, $\norm{\bw_m -\bv}$ shrinks by a constant factor.
    
    \item We then repeat this analysis across $t$ epochs (each consisting of $m$ iterations), and use a union bound. Overall, we get that after sufficiently many iterations, with high probability, the iterates get as close as we want to zero. 
\end{enumerate}

We note the optimization analysis for stochastic gradient descent is inspired by the analysis in \cite{shamir2015stochastic} for the different non-convex problem of principal component analysis (PCA), which also attempts to avoid a problematic stationary point. An interesting question for future research is to understand to what extent the polynomial dependencies in the problem parameters can be improved. 

\begin{remark} 
    Our assumption on the data that $\norm{\bx}^2 \leq c_1$ is made for simplicity. For the gradient descent case, it is easy to verify that the proof only requires that the fourth moment of the data is bounded by some constant, which ensures that the gradients of the objective function used by the algorithm are bounded. For SGD it is enough to assume that the input distribution is sub-Gaussian. The proof proceeds in the same manner, by using a variant of Azuma's inequality for martingales with sub-Gaussian tails, e.g. \cite{shamir2011variant}.
\end{remark}

\section{High-Probability Convergence}\label{sec:spherically symmetric}

The results in the previous section hold under mild conditions, but unfortunately only guarantee a constant probability of success. In this section, we consider the possibility of proving guarantees which hold with high probability (arbitrarily close to $1$). On the one hand, in Subsection \ref{subsec:spherically symmetric}, we provide such a result for the ReLU activation, assuming the input distribution $\Dcal$ is spherically symmetric. On the other hand, in Subsection \ref{subsec:failures of gaussian}, we point out non-trivial obstacles to extending such a result to non-spherically symmetric distributions. Overall, we believe that getting high-probability convergence guarantees for non-spherically symmetric distributions is an interesting avenue for future research.

\subsection{Convergence for Spherically Symmetric Distributions}\label{subsec:spherically symmetric}

In this subsection, we make the following assumptions:
\begin{assumption}\label{assump:spherically}
    Assume that:
    \begin{enumerate}
        \item $\bx \sim \mathcal{D}$ has a spherically symmetric distribution. That is, for every orthogonal matrix $A$: $A\bx\sim \mathcal{D}$
        \item The activation function $\sigma(\cdot)$ is the standard ReLU function $\sigma(z)=\max\{0,z\}$.
    \end{enumerate}
\end{assumption}

These assumptions are significantly stronger than Assumptions \ref{assump:D and sigma}, but allow us to prove a stronger high-probability convergence result. Note that even with these assumptions the loss surface is still not convex, and may contain a spurious stationary point (see Remark \ref{remark:implication on landscape}). For simplicity, we will focus on proving the result for gradient flow. The result can then be extended to gradient descent and stochastic gradient descent, along similar lines as in the proof of \thmref{thm:constant probability convergence}.

The proof strategy in this case is quite different from that of the constant-probability guarantee, and relies on the following key technical result:

\begin{lemma}\label{lem:angle decrease}
    If $\bw(t) \neq 0$, then $\frac{\partial}{\partial t}\theta(\bw(t),\bv)\leq 0$
\end{lemma}

The lemma (which relies on the spherical symmetry of the distribution) implies that if we initialize at any point  $\bw(0)\notin\text{span}\{\bv\}$, then the angle between $\bw(0)$ and $\bv$ is strictly less than $\pi$, and will remain so as long as $\bw(t)\neq 0$. As a result, we can apply \thmref{thm:innerprod} to prove that $\norm{\bw(t)-\bv}$ decays  exponentially fast. The only potential difficulty is that $\bw(t)$ may converge to the potential stationary point at the origin (at which the angle is not well-defined), but fortunately this cannot happen due to the following lemma:

\begin{lemma}\label{lem:norm gets larger}
        Let $\theta = \theta(\bw(t),\bv)$ and assume that $\bw(t)\neq 0$. If $\|\bw(t)\| \leq \max\left\{\frac{\sin(\theta) + \cos(\theta)}{2}, \frac{\sin(\theta)(1+\cos(\theta))}{2}\right\}$ then $\frac{\partial}{\partial t}\|\bw(t)\|^2 \geq 0$
\end{lemma}

The lemma can be shown to imply that as long as $\theta$ remains bounded away from $\pi$, then $\norm{\bw(t)}^2$ cannot decrease below some positive number (as its derivative is positive close enough to zero, and $\norm{\bw(t)}^2$ is a continuous function of $t$). The proof idea of both lemmas is based on a technical calculation, where we project the spherically symmetric distribution on the $2$-dimensional subspace spanned by $\bw$ and $\bv$. 

Using the lemmas above, we can get the following convergence guarantee:

\begin{theorem}\label{thm: converge w.h.p}
Assume we initialize $\bw(0)$ such that $0<\norm{\bw(0)}\leq 2$, $\theta(\bw(0),\bv) \leq \pi - \epsilon$ for some $\epsilon >0$ and that Assumption \ref{assump:D and sigma}(1) holds. Then running gradient flow, we have for all $t\geq 0$
\[
\norm{\bw(t) - \bv}^2 \leq \norm{\bw(0)-\bv}\exp(-\lambda t)
\]
where $\lambda = \frac{\alpha^4\beta}{8\sqrt{2}}\sin^3\left(\frac{\epsilon}{8}\right)$.
\end{theorem}

We now note that the assumption of the theorem holds with exponentially high probability under standard initialization schemes. For example, if we use a Gaussian initialization $\bw(0)\sim \mathcal{N}(0,\frac{1}{d}I)$, then by standard concentration of measure arguments, it holds  w.p $> 1-e^{-\Omega(d)}$ that $\theta(\bw(0),\bv)$ is at most (say) $\frac{3\pi}{4}$, and w.p $> 1-e^{-\Omega(d)}$ that $\norm{\bw(0)} \leq 2$. As a result, by \thmref{thm: converge w.h.p}, w.p $> 1-e^{-\Omega(d)}$ over the initialization we have $\norm{\bw(t)-\bv}^2 \leq \norm{\bw(0)-\bv}^2 e^{-\Omega(t)}$ for all $t$. The full proof of the theorem can be found in \appref{appen:proofs from spherically}.

\begin{remark}
If we further assume that the distribution is a standard Gaussian, then it is possible to prove \lemref{lem:angle decrease} and \lemref{lem:norm gets larger} in a much easier fashion. The reason is that specifically for a standard Gaussian distribution there is a closed-form expression (without the expectation) for the loss and the gradient, see \cite{brutzkus2017globally}, \cite{safran2017spurious}. We provide the relevant versions of the lemmas, as well as their proofs, in \subsecref{subsec:standard gaussian proofs}.
\end{remark}

\subsection{Non-monotonic Angle  Behavior}\label{subsec:failures of gaussian}

\begin{figure}[t]
{\includegraphics[width=3in]{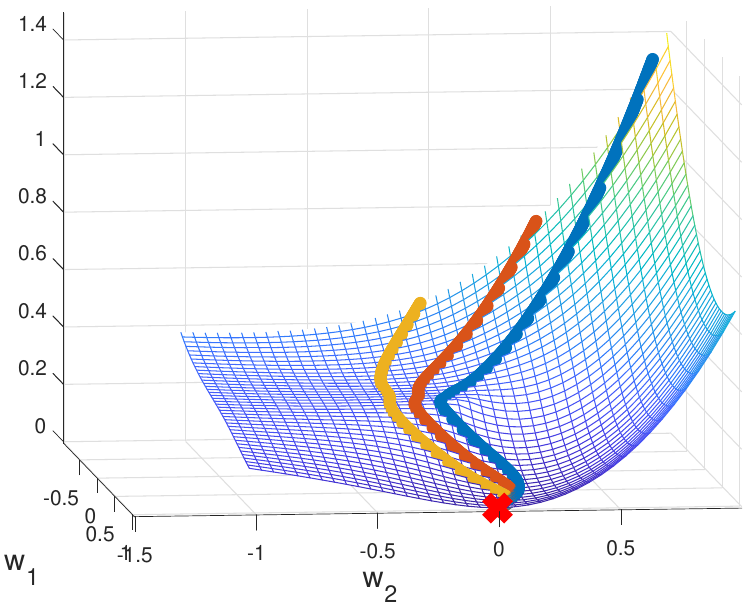}}
{\includegraphics[width = 3.5in]{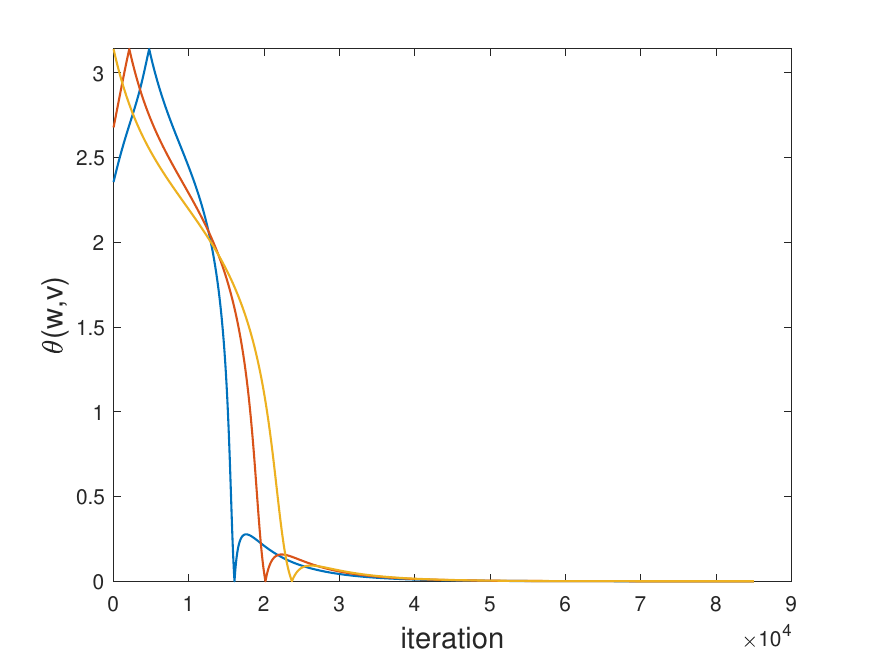}}
\caption{Gradient descent for $2$-dimensional data (best viewed in color). The left figure represents the trajectory of gradient descent over the loss surface. The red "x" marker represents the global minimum at $\bw=\bv=(1,0)$. The right figure shows the angle between $\bw$ and $\bv$ as a function of the number of iterations, where the angle ranges from $0$ to $\pi$. The plot colors in the right figure correspond to the trajectory colors in the left figure.}
\label{fig:gd}
\end{figure}

The results in the previous subsection crucially relied on the fact that at almost any point $\bw$, the angle $\theta(\bw,\bv)$ decreases. This type of analysis was also utilized in works on related settings (e.g., \citet{brutzkus2017globally}).

Based on this, it might be tempting to conjecture that this monotonically decreasing angle property (and as a result, high-probability guarantees) can be shown to hold more generally, not just for symmetrically spherical distributions. Perhaps surprisingly, we show empirically that this may not be the case, already when we discuss the simple setting of unit variance Gaussian with a \emph{non-zero} mean. We emphasize that this does not necessarily mean that gradient methods will not succeed, only that an analysis based on showing monotonic behavior of the relevant geometric quantity will not work in general. 

In particular, in Figure \ref{fig:gd} we report the result of running gradient descent (with constant step size $\eta=10^{-3}$) on our objective function $F$ in $\reals^2$, where the input distribution $\mathcal{D}$ is a unit-variance Gaussian with mean at $(0,1)$, and our target vector is $\bv = (1,0)$. We initialize at three different locations: $w_1 = (-1~1),~ w_2 = (-1,0.5),~w_3 = (-1,0)$. Although the algorithm eventually reaches the global minimum $\bw=\bv$, the angle between them is clearly non-monotonic, and actually is initially increasing rather than decreasing. Even worse, the angle appears to attain every value in $(0,\pi]$, so it appears that any analysis using angle-based ``safe regions'' is bound to fail. 

Overall, we conclude that proving a high-probability convergence guarantee for gradient methods appears to be an interesting open problem, already in the case of unit-variance, non-zero-mean Gaussian input distributions. We leave tackling this problem to future work.

\vskip 0.5cm
\textbf{Acknowledgements.} This research is supported in part by European Research Council (ERC) grant 754705. We thank Itay Safran for spotting a bug in the proof of \thmref{thm:innerprod}.

\setcitestyle{numbers}
\bibliographystyle{abbrvnat}
\bibliography{mybib}

\appendix

\section{Proofs from \secref{sec:assumptions}}\label{appen:proof from sec assumptions}

\begin{proof}[Proof of \thmref{thm:too strong activation assumption}]
We have that:

\begin{align*}
	\inner{\nabla F(\bw),\bw-\bv}&=\E_{\bx}\left[(\sigma(\bw^\top\bx)-\sigma(\bv^\top\bx))\sigma'(\bw^\top\bx)(\bw^\top\bx-\bv^\top\bx)\right]\\
	&\stackrel{(*)}{=} \E_{\bx}\left[\gamma\cdot (\sigma(\bw^\top\bx)-\sigma(\bv^\top\bx))(\bw^\top\bx-\bv^\top\bx)\right]\\
	&\stackrel{(**)}{=} \E_{\bx}\left[\gamma^2(\bw^\top\bx-\bv^\top\bx)^2\right]~=~
	\gamma^2(\bw-\bv)^\top \Sigma (\bw-\bv) \geq \gamma^2 \lambda \norm{\bw-\bv}^2
\end{align*}
where $(*)$ is by monotonicity of $\sigma$ (hence $(\sigma(\bw^\top\bx)-\sigma(\bv^\top\bx))(\bw^\top\bx-\bv^\top\bx)\geq 0$ always), and $(**)$ is by the assumption that $\sigma'(z)\geq \gamma$. Next, we bound the gradient $\nabla F (\bw)$:
 \begin{align*}
   \norm{\nabla F(\bw_t)}^2 & = \mathbb{E}_\bx\left[ \left( \sigma(\bw_t^\top\bx) - \sigma(\bv^\top\bx) \right)^2\cdot \sigma'(\bw^\top\bx)^2 \bx^\top\bx \right] \\
    & \leq c_2^4\mathbb{E}_\bx\left[ \left( \bw_t^\top\bx - \bv^\top\bx \right)^2\cdot \bx^\top\bx \right] \\
    & \leq c_2^4 \norm{\bw_t-\bv}^2\mathbb{E}_\bx\left[ \norm{\bx}^2\cdot \bx^\top\bx \right] \leq c_1^2c_2^4 \norm{\bw_t-\bv}^2.
\end{align*}

At iteration $t+1$ we have that:
\begin{align*}
    \norm{\bw_{t+1} - \bv}^2 &= \norm{\bw_t - \eta \nabla F (\bw_t) - \bv}^2 \\
    & = \norm{\bw_t - \bv}^2 - 2\eta \inner{\nabla F (\bw_t), \bw_t - \bv} + \eta^2 \norm{\nabla F(\bw_t)}^2 \\
    &\leq \norm{\bw_t - \bv}^2 - 2\gamma^2\lambda\eta \norm{\bw_t - \bv}^2 + \eta^2 c_1^2c_2^4 \norm{\bw_t - \bv}^2 \\
    & \leq \norm{\bw_t - \bv}^2 \left( 1 - \gamma^2\lambda\eta \right).
\end{align*}
Using induction over the above proves the lemma.

\end{proof}

\section{Proofs from \secref{sec:mild assumptions gradient}}\label{appen:proofs from sec mild}
We will first need the following lemma:
\begin{lemma}\label{lem:pie_slice}
	Fix some $\alpha\geq 0$, and let $\ba,\bb$ be two vectors in $\reals^2$ such that $\theta(\ba,\bb)\leq \pi-\delta$ for some $\delta\in (0,\pi]$. Then
	\[
	\inf_{\bu:\norm{\bu}=1}\int 
	\mathbbm{1}_{\ba^\top\by>0}\mathbbm{1}_{\bb^\top\by>0}\mathbbm{1}_{\norm{\by}\leq 
		\alpha}(\bu^\top \by)^2d\by ~\geq~ \frac{\alpha^4}{8\sqrt{2}}\sin^3\left(\frac{\delta}{4}\right)~.
	\]
\end{lemma}

\begin{proof} 
	It is enough to lower bound
	\[
	\inf_{\bu}~~\inf_{\bb:\theta(\ba,\bb)\leq \pi-\delta}\int 
	\mathbbm{1}_{\ba^\top\by>0,\bb^\top\by>0,\norm{\by}\leq 
		\alpha}(\bar{\bu}^\top \by)^2d\by~.
	\]
	The inner infimum is attained at some $\bb$ such that 
	$\theta(\ba,\bb)=\pi-\delta$. This is because $\bar{\bu}^\top\by$ does not depend on $\ba$ and $\bb$, and the volume for which the indicator function inside the integral is non-zero is smallest when the angle $\theta(\ba,\bb)$ is largest. Setting this and switching 
	the order of the infima, we get
	\[
	\inf_{\bb:\theta(\ba,\bb)=-\pi+\delta}~~\inf_{\bu}\int 
	\mathbbm{1}_{\ba^\top\by>0}\mathbbm{1}_{\bb^\top\by>0}\mathbbm{1}_{\norm{\by}\leq 
		\alpha}(\bar{\bu}^\top \by)^2d\by~.
	\]
	When $\theta(\ba,\bb)=-\pi+\delta$, we note that the set $\{\by\in\reals^2:\ba^\top\by>0,\bb^\top\by>0,\norm{\by}\leq \alpha\}$ is simply a ``pie slice'' of radial width $\delta$ out of a ball of radius $\alpha$. Since the expression is invariant to rotating the coordinates, we will consider without loss of generality the set $P=\{\by:\theta(\by,\be_1)\leq \delta/2,\norm{\by}\leq \alpha\}$, and the expression above reduces to
	\begin{align}
	\inf_{\bu}\int_{\by\in P}(\bar{\bu}^\top\by)^2d\by~&=~\inf_{\bu:\norm{\bu}=1}\int_{\by\in P}
	\left((u_1y_1)^2+(u_2y_2)^2+2u_1u_2y_1y_2\right)d\by\notag\\
	~&\stackrel{(*)}{=}~
	\inf_{\bu:\norm{\bu}=1}\int_{\by\in P}
	\left((u_1y_1)^2+(u_2y_2)^2\right)d\by\notag\\
	~&=~
	\inf_{u_1,u_2:u_1^2+u2^2=1}u_1^2\int_{\by\in P}y_1^2d\by + u_2^2\int_{\by\in P}y_2^2d\by~=~
	\min\left\{\int_{\by\in P}y_1^2d\by~,~\int_{\by\in P}y_2^2d\by\right\}\notag\\
	&~\geq~\int_{\by\in P}\min\{y_1^2,y_2^2\}d\by~,\label{eq:minyy}
	\end{align}
	where $(*)$ is from the fact that $P$ is symmetric around the $x$-axis (namely, $(y_1,y_2)\in P$ if and only if $(y_1,-y_2)\in P$). 

	\begin{figure}[t]
	    \centering
        \includegraphics[scale=0.8]{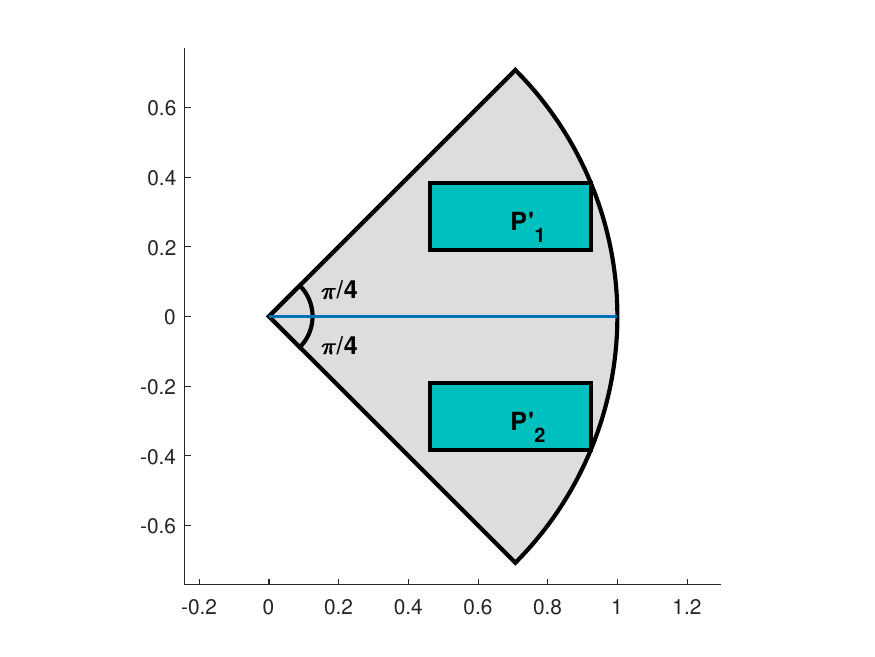}
	    \caption{An illustration of the sets $P,P'_1, P'_2$ for the case of $\alpha =1,~ \delta = \frac{\pi}{2}$. The set $P$, colored in gray, is a "pie slice" and the rectangles $P'_1, P'_2$ are contained in $P$.}
	    \label{fig:P}
	\end{figure}

	We now note that the set $P$ contains the two (disjoint and equally-sized) rectangular sets
	\[
	P'_1:=\left[\frac{\alpha}{2}\cos\left(\frac{\delta}{4}\right),\alpha\cos\left(\frac{\delta}{4}\right)\right]\times
	\left[\frac{\alpha}{2}\sin\left(\frac{\delta}{4}\right),\alpha\sin\left(\frac{\delta}{4}\right)\right]
	\]
	and
	\[
	P'_2:=\left[\frac{\alpha}{2}\cos\left(\frac{\delta}{4}\right),\alpha\cos\left(\frac{\delta}{4}\right)\right]\times
	\left[-\alpha\sin\left(\frac{\delta}{4}\right),-\frac{\alpha}{2}\sin\left(\frac{\delta}{4}\right)\right]
	\]	
	(see Figure \ref{fig:P} for an illustration).
	Therefore, we can lower bound \eqref{eq:minyy} by
	\begin{align*}
	\int_{\by\in P'_1\cup P'_2} \min\{y_1^2,y_2^2\}d\by
	~&=~\left(\min_{\by\in P'_1\cup P'_2}\min\{y_1^2,y_2^2\}\right)\int_{\by\in P'_1\cup P'_2}1 d\by\\
	&=~\frac{\alpha^2}{4}\min\left\{\cos^2\left(\frac{\delta}{4}\right),\sin^2\left(\frac{\delta}{4}\right)\right\}\cdot \int_{\by\in P'_1\cup P'_2}1d\by\\
	&=~\frac{\alpha^2}{4}\sin^2\left(\frac{\delta}{4}\right)\cdot \int_{\by\in P'_1\cup P'_2}1d\by~,
	\end{align*}
	where we used the fact that $\frac{\delta}{4}\in \left[0,\frac{\pi}{4}\right]$ and therefore $\cos^2(\delta/4)\geq \sin^2(\delta/4)$. 
	The integral is simply the volume of $P'_1\cup P'_2$, and since $P'_1$ and $P'_2$ are disjoint and equally sized rectanges, this equals twice the volume of $P'_1$, namely
	$
	2\cdot\frac{\alpha}{2}\cos\left(\frac{\delta}{4}\right)\cdot \frac{\alpha}{2}\sin\left(\frac{\delta}{4}\right)
	$. Plugging into the above, we get
	\[
	\frac{\alpha^2}{4}\sin^2\left(\frac{\delta}{4}\right)\cdot \frac{\alpha^2}{2}\cos\left(\frac{\delta}{4}\right)\sin\left(\frac{\delta}{4}\right)~=~
	\frac{\alpha^4}{8}\sin^3\left(\frac{\delta}{4}\right)\cos\left(\frac{\delta}{4}\right)~\geq \frac{\alpha^4}{8\sqrt{2}}\sin^3\left(\frac{\delta}{4}\right)~,
	\]
	where again we used the fact that $\delta/4 \in [0,\pi/4]$. 

\end{proof}

We now turn to prove the theorem:
\begin{proof}[Proof of \thmref{thm:innerprod}]
    
We have:
\begin{equation}\label{eq:product}
\inner{\nabla F(\bw),\bw-\bv} = 
\E_{\bx}\left[\left(\sigma(\bw^\top\bx)-\sigma(\bv^\top\bx)\right)
\cdot\sigma'(\bw^\top\bx)\cdot(\bw^\top\bx-\bv^\top\bx)\right]~.
\end{equation}

Let $P$ be the orthogonal projection on the plane spanned by $\bw$ and $\bv$.
We note that since $\sigma$ is monotonically non-decreasing, then for any $\bx$, 
$\sigma'(\bw^\top\bx)\geq 0$ and $(\sigma(\bw^\top 
\bx)-\sigma(\bv^\top))(\bw^\top\bx-\bv^\top\bx)\geq 0$. As 
a result, we can lower bound \eqref{eq:product} by
\begin{align*}
&\E_{\bx}\left[\mathbbm{1}_{\bw^\top\bx>0}\mathbbm{1}_{\bv^\top\bx>0}
\left(\sigma(\bw^\top\bx)-\sigma(\bv^\top\bx)\right)
\cdot\sigma'(\bw^\top\bx)\cdot(\bw^\top\bx-\bv^\top\bx) \right]\\
&\geq
\E_{\bx}\left[\mathbbm{1}_{\norm{P\bx} \leq \alpha}\mathbbm{1}_{\bw^\top\bx>0}\mathbbm{1}_{\bv^\top\bx>0}
\left(\sigma(\bw^\top\bx)-\sigma(\bv^\top\bx)\right)
\cdot\gamma\cdot(\bw^\top\bx-\bv^\top\bx) \right]\\
&=\gamma\cdot\E_{\bx}\left[\mathbbm{1}_{\norm{P\bx} \leq \alpha}\mathbbm{1}_{\bw^\top\bx>0}\mathbbm{1}_{\bv^\top\bx>0}
\left(\sigma(\bw^\top\bx)-\sigma(\bv^\top\bx)\right)
(\bw^\top\bx-\bv^\top\bx) \right]~,
\end{align*}
where we used that $\norm{\bw} \leq 2$, hence for $\norm{\bx}\leq \alpha$ 
(and also for $\norm{P\bx}\leq \alpha$, since $P$ is an orthogonal projection) 
we have  $\inner{\bx,\bw} \leq 2\alpha$ which by our assumption means that $\sigma'(\inner{\bw,\bx}) > \gamma$. By the assumption that $\sigma'(z)\geq \gamma$ for any $0<z<2\alpha$, it follows 
that $(\sigma(z')-\sigma(z))\cdot(z'-z)\geq \gamma(z'-z)^2$ for any $0< z,z'<2\alpha$ 
As a result, the displayed equation above is at least
\begin{align*}
&\gamma^2\cdot \E_{\bx}\left[\mathbbm{1}_{\norm{P\bx} \leq \alpha}\mathbbm{1}_{\bw^\top>0}\mathbbm{1}_{\bv^\top\bx>0}
(\bw^\top \bx-\bv^\top\bx)^2\right]\\
&=
\gamma^2\norm{\bw-\bv}^2\cdot 
\E_{\bx}\left[\mathbbm{1}_{\norm{P\bx} \leq \alpha}\mathbbm{1}_{\bw^\top\bx>0}\mathbbm{1}_{\bv^\top\bx>0}
((\overline{\bw-\bv})^\top\bx)^2\right]\\
&\geq
\gamma^2\norm{\bw-\bv}^2\cdot \inf_{\bu\in \text{span}\{\bw,\bv\},\norm{\bu}=1} 
\E_{\bx}\left[\mathbbm{1}_{\norm{P\bx} \leq \alpha}\mathbbm{1}_{\bw^\top\bx>0}\mathbbm{1}_{\bv^\top\bx>0}(\bu^\top\bx)^2\right]
\end{align*}
Since the expression inside the expectation above depends just on inner 
products of $\bx$ with $\bw,\bv$, we can consider the marginal distribution 
$\Dcal_{\bw,\bv}$ of $\bx$ on the $2$-dimensional subspace spanned by $\bw,\bv$ 
(with density function 
$p_{\bw,\bv}$), and letting $\hat{\bw},\hat{\bv}$ denote the projections of 
$\bw,\bv$ on that subspace, write the above as
\begin{align*}
& \gamma^2\norm{\bw-\bv}^2\cdot \inf_{\bu\in \reals^2,\norm{\bu}=1}
\E_{\by\sim 
	\Dcal_{\bw,\bv}}\left[\mathbbm{1}_{\hat{\bw}^\top\by>0}\mathbbm{1}_{\hat{\bv}^\top\by>0}\mathbbm{1}_{\norm{\by}\leq
	\alpha}(\bu^\top\by)^2\right]\\
&=
\gamma^2\norm{\bw-\bv}^2\cdot \inf_{\bu\in \reals^2,\norm{\bu}=1}
\int
\mathbbm{1}_{\hat{\bw}^\top\by>0}\mathbbm{1}_{\hat{\bv}^\top\by>0}\mathbbm{1}_{\norm{\by}\leq
	\alpha}(\bu^\top\by)^2p_{\bw,\bv}(\by)d\by\\
&\geq \beta\gamma^2\norm{\bw-\bv}^2\cdot \inf_{\bu\in \reals^2,\norm{\bu}=1}
\int
\mathbbm{1}_{\hat{\bw}^\top\by>0}\mathbbm{1}_{\hat{\bv}^\top\by>0}\mathbbm{1}_{\norm{\by}\leq
	\alpha}(\bu^\top\by)^2d\by~,
\end{align*}
where the last step is by our assumptions (note that if $\bw=\bv$, the theorem 
statement is trivially true by \eqref{eq:product} which implies that the inner 
product is non-negative). The theorem now follows from 
\lemref{lem:pie_slice}.
\end{proof}

\section{Proofs from \secref{sec:convergence with constant probability}}\label{appen:proofs from constant probability}

\begin{proof}[Proof of \lemref{lemma:constant probability initialization}]
    Fix some $\epsilon > 0$ to be determined later. We have that:
    \begin{align*}
        & \mathcal{P}\left(\|\bw-\bv\|^2 \leq 1-\epsilon\right) = \mathcal{P}\left(\|\bw\|^2 - 2\inner{\bw,\bv} \leq -\epsilon\right) \\
        & = \mathcal{P}\left(\inner{\bw,\bv} \geq \frac{\|\bw\|^2 + \epsilon}{2}\right).
    \end{align*}
    Since the distribution of $\bw$ is spherically symmetric, we can assume w.l.o.g that $\bv=(1,0)$, so that $\inner{\bw,\bv}=w_1$. Thus, the above probability can be written as:
    \begin{align}\label{eq:probability bound}
        &\mathcal{P}\left(\inner{\bw,\bv} \geq \frac{\|\bw\|^2 + \epsilon}{2}\right) = \mathcal{P}\left(w_1 \geq \frac{\norm{\bw}^2 + \epsilon}{2}\right) \nonumber\\
        ~\geq~ &\mathcal{P}\left(w_1 \geq 2\mathbb{E}\left[\norm{\bw}^2\right]\right) - \mathcal{P}\left(\frac{\norm{\bw}^2 + \epsilon}{2} \geq 2\mathbb{E}\left[\norm{\bw}^2\right]\right)
    \end{align}
    where we used the fact that for every two random variable $A,B$ and constant $c$ we have that $\mathcal{P}(A \geq B) \geq \mathcal{P}(A\geq c) - \mathcal{P}(B \geq c)$. For the first term of \eqref{eq:probability bound}, we know that $\mathbb{E}\left[\norm{\bw}^2\right] = \tau^2 d$, hence:
    \begin{align*}
        \mathcal{P}\left(w_1 \geq 2\mathbb{E}\left[\norm{\bw}^2\right]\right) = \mathcal{P}\left(w_1 \geq 2 \tau^2 d \right) = \frac{1}{2} - \frac{1}{2} \text{erf} \left( \sqrt{2}\tau d \right)
    \end{align*}
    where $\text{erf}$ is the error function. For any $0 < z < 1$ it can be easily verified that $\text{erf}(z) \geq \frac{z}{3}$.
    Combining this and using the assumption that $\tau \leq \frac{1}{d\sqrt{2}}$ we can bound :
    \[
     \mathcal{P}\left(w_1 \geq 2\mathbb{E}\left[\norm{\bw}^2\right]\right) \geq \frac{1}{2} - \frac{1}{3\sqrt{2}}\tau d \geq \frac{1}{2} - \frac{1}{4}\tau d
    \]
    For the second term of \eqref{eq:probability bound} take $\epsilon = 2\tau^2 d$ to get:
    \begin{align*}
        \mathcal{P}\left(\frac{\norm{\bw}^2 + \epsilon}{2} \geq 2\mathbb{E}\left[\norm{\bw}^2\right]\right) &= \mathcal{P}\left( \norm{\bw}^2 \geq 4\tau^2 d - \epsilon \right) \\
        &  \leq \mathcal{P}\left( \norm{\bw}^2 \geq 2\tau^2 d\right) \leq \left(2e^{-1}\right)^{d/2} \leq 1.2^{-d}
    \end{align*}
    where in the second inequality we used a standard tail bound on Chi-squared distributions. Combining the above with \eqref{eq:probability bound} we get that:
    \begin{equation*}
        \mathcal{P}\left(\|\bw-\bv\|^2 \leq 1-2\tau^2d\right) \geq \frac{1}{2} - \frac{1}{4}\tau d - 1.2^{-d}.
    \end{equation*}
\end{proof}
\subsection{Gradient Flow}
\begin{proof}[Proof of \thmref{thm:constant probability convergence}(1)]
    First we show that at every time $t_0$ for which $\norm{\bw(t_0) -\bv} < 1$ the conditions of \thmref{thm:innerprod} hold. We have that $\norm{\bw(t_0)} \leq \norm{\bw(t_0) -\bv} + \norm{\bv} < 2$, hence $\norm{\bw(t_0)} <2$. Next $\norm{\bw(t_0)-\bv}^2 < 1$ and $\norm{\bv}^2=1$ hence $\inner{\bw(t_0),\bv} \geq \frac{1}{2}\norm{\bw(t_0)}^2 >0$ which means that $\theta(\bw(t_0),\bv) < \frac{\pi}{2}$.
    This shows that we can use \thmref{thm:innerprod} at time $t=t_0$ to get that:
    \begin{align}\label{eq:grad norm w v}
        \frac{\partial}{\partial t}\|\bw(t)-\bv\|^2 &= 2\inner{\bw(t)-\bv,\frac{\partial}{\partial t} \bw(t)} = -2\inner{\bw(t)-\bv,\nabla F(\bw(t))} \leq 0.
    \end{align}
    By the assumptions of the theorem, the above holds for time $t_0=0$. Assume on the way of contradiction that for some time $t>0$ we have that $\norm{\bw(t)-\bv} \geq 1$, and let $t_1$ be the first time that this happens. Then for every $t_0<t<t_1$ we have that $\norm{\bw(t)-\bv} <1$. But because $\norm{\bw(t_1)-\bv} \geq 1$ we have that for some time $t_0<t<t_1$: $\frac{\partial}{\partial t}\norm{\bw(t)-\bv} > 0$, a contradiction to \eqref{eq:grad norm w v}. Hence for every $t\geq 0$ we have that $\norm{\bw(t)-\bv} <1$ and the conditions of \thmref{thm:innerprod} hold.
    
    Using \thmref{thm:innerprod} again we get that for every $t>0$:
    \begin{equation*}
        \inner{\nabla F(\bw(t)), \bw(t) - \bv} \geq \frac{\alpha^4\beta\gamma^2}{8\sqrt{2}}\sin\left(\frac{\pi}{8}\right)^3\|\bw(t)-\bv\|^2 \geq \frac{\alpha^4\beta\gamma^2}{210}|\bw(t)-\bv\|^2.
    \end{equation*} 
    Set $\lambda = \frac{\alpha^4\beta\gamma^2}{210}$, in total we have that:
    \begin{equation*}
        \frac{\partial}{\partial t} \|\bw(t) - \bv\|^2 = -2\inner{\nabla F(\bw(t)),\bw(t)-\bv} \leq -\lambda\|\bw(t)-\bv\|^2~.
    \end{equation*}
    Using Gr\"{o}nwall's inequality, this proves that for every $t>0$ we get:
    \begin{equation*}
        \|\bw(t) - \bv\|^2 \leq \|\bw(0)-\bv\|^2\exp(-\lambda t).
    \end{equation*}

\end{proof}

\subsection{Gradient Descent}

\begin{proof}[Proof of \thmref{thm:constant probability convergence}(2)]
    Assume that $\norm{\bw_t - \bv}^2 < 1$ for some $t\geq 0$, then we have that $\theta(\bw_t,\bv)\leq \frac{\pi}{2}$. Thus, we can use \thmref{thm:innerprod} with $\delta=\frac{\pi}{2}$ to get that:
    \begin{align*}
       \|\bw_{t+1} - \bv\|^2 & = \|\bw_{t} - \eta \nabla F(\bw_t) - \bv\|^2 \\
        & = \norm{\bw_t - \bv}^2 - 2\eta \inner{\nabla F(\bw_t),\bw_t - \bv} +\eta^2 \norm{\nabla F(\bw_t)}^2 \\
        & \leq \norm{\bw_t - \bv}^2(1 - \eta\lambda) + \eta^2 \norm{\nabla F(\bw_t)}^2.
    \end{align*}
    Now to bound the second term of the above expression recall the definition of $\nabla F(\bw_t)$ to get:
    \begin{align*}
       \norm{\nabla F(\bw_t)}^2 & = \mathbb{E}_\bx\left[ \left( \sigma(\bw_t^\top\bx) - \sigma(\bv^\top\bx) \right)^2\cdot \sigma'(\bw^\top\bx)^2 \bx^\top\bx \right] \\
        & \leq c_2^4\mathbb{E}_\bx\left[ \left( \bw_t^\top\bx - \bv^\top\bx \right)^2\cdot \bx^\top\bx \right] \\
        & \leq c_2^4 \norm{\bw_t-\bv}^2\mathbb{E}_\bx\left[ \norm{\bx}^2\cdot \bx^\top\bx \right] \leq c_1^2c_2^4 \norm{\bw_t-\bv}^2
    \end{align*}
    where in the first inequality we used that $\sigma$ is monotonic with bounded derivative, and in the second inequality we used Cauchy-Schwartz. Note that by our choice of $\eta$: 
      \begin{equation*}
          1-\eta\lambda + \eta^2c <1 - \frac{\eta\lambda}{2} < 1,
      \end{equation*} 
    this proves that:
    \begin{equation}\label{eq:gradient descent recursive}
         \|\bw_{t+1} - \bv\|^2 \leq (1-\eta\lambda + \eta^2c)\norm{\bw_t - \bv}^2 \leq \left(1-\frac{\eta\lambda}{2}\right)\norm{\bw_t - \bv}^2
    \end{equation}
    and in particular $\norm{\bw_{t+1}-\bv} < 1$. Now after $T$ iterations we can use \eqref{eq:gradient descent recursive} iteratively to get that:
      \begin{align*}
          \|\bw_{T} - \bv\|^2 &~\leq~ \left(1-\frac{\eta\lambda}{2}\right)\|\bw_{T-1} - \bv\|^2\\ & ~\leq~ ...~\leq~ \left(1-\frac{\eta\lambda}{2}\right)^{T}\|\bw_{0} - \bv\|^2~.
      \end{align*}
\end{proof}

\subsection{Stochastic Gradient Descent}

First, we prove a recursion relation similar to the one in the gradient descent step. Only here since each gradient step is stochastic we can only prove that the recursion relation holds in expectation over the example selected in each iteration.

\begin{lemma}\label{lemma:recursion norm}
    Suppose that $\|\bw_{t} - \bv\|^2\leq 1-\epsilon$. Then
    \begin{equation*}
       \mathbb{E}\left[\|\bw_{t+1} - \bv\|^2|\bw_t\right] \leq (1-2\eta\lambda+\eta^2c)\|\bw_{t} - \bv\|^2
    \end{equation*}
    where $c = c_1^2c_2^4$.
\end{lemma}

\begin{proof}
    We can use \thmref{thm:innerprod} with $\delta=\frac{\pi}{2}$ to get that
    \begin{align*}
        \mathbb{E}\left[\|\bw_{t+1} - \bv\|^2|\bw_t\right] & = \mathbb{E}\left[\|\bw_{t} - \eta g_t - \bv\|^2|\bw_t\right] \\
        & = \norm{\bw_t - \bv}^2 - 2\eta \mathbb{E}[\inner{g_t,\bw_t - \bv}|\bw_t] + \eta^2 \mathbb{E}[\norm{g_t}^2|\bw_t] \\
        & = \norm{\bw_t - \bv}^2 - 2\eta \inner{\nabla F(\bw_t),\bw_t - \bv} +\eta^2 \norm{\nabla F(\bw_t)}^2 \\
        & \leq \norm{\bw_t - \bv}^2(1 - 2\eta\lambda) + \eta^2 \norm{\nabla F(\bw_t)}^2~.
    \end{align*}
    Now to bound the second term recall the definition of $\nabla F(\bw_t)$ to get:
    \begin{align*}
       \norm{\nabla F(\bw_t)}^2 & = \mathbb{E}_\bx\left[ \left( \sigma(\bw_t^\top\bx) - \sigma(\bv^\top\bx) \right)^2\cdot \sigma'(\bw^\top\bx)^2 \bx^\top\bx \right] \\
        & \leq c_2^4\mathbb{E}_\bx\left[ \left( \bw_t^\top\bx - \bv^\top\bx \right)^2\cdot \bx^\top\bx \right] \\
        & \leq c_2^4 \norm{\bw_t-\bv}^2\mathbb{E}_\bx\left[ \norm{\bx}^2\cdot \bx^\top\bx \right] \leq c_1^2c_2^4 \norm{\bw_t-\bv}^2
    \end{align*}
    where in the first inequality we used that $\sigma$ is monotonic with bounded derivative, and in the second inequality we used Cauchy-Schwartz. This proves the required bound.
\end{proof}

The recursion relation above only works if $\bw_t$ is in a "safe zone", that is $\norm{\bw_t-\bv}^2 \leq 1-\epsilon$. Although in expectation the distance between $\bw_t$ and $\bv$ only decrease, taking a stochastic step may take $\bw_{t+1}$ outside of the safe zone. The following lemma shows that if $\eta$ is small enough, then taking at most $m = O(1/\eta)$ steps keeps $\bw_t$ in the "safe zone" w.h.p for every $t=1,\dots,m$. 

\begin{lemma}\label{lem:w.h.p m iterations w i close to v}
    Assume that $\norm{\bw_0 - \bv}^2 \leq 1-\epsilon$, and Let $\delta >0$. Then w.p $> 1-\delta$, if $\eta <\frac{\epsilon^2\lambda}{3c_1^2c_2^4\log\left(\frac{1}{\delta}\right)}$ and $m \leq \frac{1}{9\eta c_1c_2^2}$ then for every $i=1,\dots,m$ we have that $\norm{\bw_i - \bv}^2 \leq 1-\frac{\epsilon}{2}$.

\end{lemma}
\begin{proof}
    Denote $X_i = \norm{\bw_i-\bv}^2$, then we have:
    \begin{align}\label{eq:martingale X bound}
        \left|X_i - X_{i-1}\right| &= \left|\norm{\bw_i-\bv}^2 - \norm{\bw_{i-1}-\bv}^2\right| =  \left|\norm{\bw_{i-1} - \eta \bg_{i-1} -\bv}^2 - \norm{\bw_{i-1}-\bv}^2\right| \nonumber\\
        & = \left| -2\eta \inner{\bg_{i-1},\bw_{i-1}-\bv} + \eta^2\norm{\bg_{i-1}}^2\right| \leq 2\eta |\inner{\bg_{i-1},\bw_{i-1} - \bv}| + \eta^2\norm{\bg_{i-1}}^2
    \end{align}
    We will bound the norm of the gradient at each step:
    \begin{align*}
        \norm{\bg_i}^2 = \bx_i^\top \bx_i \sigma'\left(\bw_i^\top \bx_i\right)^2\left(\sigma\left(\bw_i^\top\bx_i\right) - \sigma\left(\bv^\top\bx_i\right) \right)^2 \leq c_1^2c_2^4\norm{\bw_i -\bv}^2 
    \end{align*}
    thus we can bound \eqref{eq:martingale X bound} with:
    \begin{align}\label{eq:martingale X second bound}
        |X_i - X_{i-1}| &\leq \norm{\bw_{i-1}-\bv}^2c_1^2c_2^4(2\eta + \eta^2) \leq 3\eta c_1^2c_2^4\norm{\bw_{i-1}-\bv}^2
    \end{align}
    Denote $\eta'=3\eta c_1^2c_2^4$. Using \eqref{eq:martingale X bound} we can bound:
    \begin{equation}\label{eq:bound distance w i v}
    \norm{\bw_i -\bv}^2 \leq \norm{\bw_{i-1} - \bv}^2 + \eta' \norm{\bw_{i-1} - \bv}^2 \leq (1 + \eta')\norm{\bw_{i-1} - \bv}^2    
    \end{equation}
    Thus, combining \eqref{eq:martingale X second bound} and \eqref{eq:bound distance w i v} we get:
    \begin{align*}
    |X_i - X_{i-1}| &\leq \eta'(1 + \eta')\norm{\bw_{i-2} - \bv}^2 \\
    & \leq ...\leq \eta'(1 + \eta')^{i-2}\norm{\bw_{0} - \bv}^2 \leq \eta' (1 + \eta')^i(1-\epsilon)  
    \end{align*}
    We would like to use Azuma's inequality on $X_i$, but in order to prove that they are supermartingales we need to use \lemref{lemma:recursion norm}. The problem here is that the condition of the lemma, that $\norm{\bw_t-\bv}^2 < 1-\epsilon$, does not necessarily holds, hence the series $X_i$ may not be supermartingales. Instead, we consider a dual series of random variables $\tilde{X}_i = \min\left\{X_i, 1-\frac{\epsilon}{2}\right\}$, and prove that they are supermartingales. First we have that:
    \[
    \left| \tilde{X}_i - \tilde{X}_{i-1}\right| \leq |X_i - X_{i-1}| \leq \eta'(1+\eta')^i(1-\epsilon).
    \]
    Next, we have for every $i$ that $\tilde{X}_i \leq 1-\frac{\epsilon}{2}$, thus we can use \lemref{lemma:recursion norm} (note that the result of the lemma does not depend on the value of $\epsilon$) and choose $\eta' \leq \frac{\lambda}{c_1^2c_2^4}$ to get that:
    \[
    \mathbb{E}[\tilde{X}_i|\bw_{i-1}] \leq \min\{(1-2\eta'\lambda + \eta'^2c_1^2c_2^4)X_{i-1}, 1-\epsilon \} \leq \tilde{X}_{i-1}
    \]
    this proves that the series $\tilde{X}_i$ are supermartingales. Now we use a maximal version of Azuma-Hoeffding inequality (see \cite{hoeffding1994probability}) on $\tilde{X}_i$ to show that after $m$ iterations we have that:
    \begin{align}\label{eq:hoeffding bound expectation}
        \mathcal{P}\left( \sup_{1\leq i\leq m}\tilde{X}_i - \tilde{X}_0 > \frac{\epsilon}{2}  \right) &\leq \exp \left( \frac{-\epsilon^2}{2\sum_{i=0}^m \left(\eta'(1+\eta')^i(1-\epsilon)\right)^2} \right) \nonumber\\
        & \leq\exp \left( \frac{-\epsilon^2}{2\eta'^2(1-\epsilon)^2\frac{(1+\eta')^{2m+2} - 1}{(1+\eta')^2 - 1}} \right) \nonumber\\
        &\leq \exp \left( \frac{-\epsilon^2}{2\eta'^2(1-\epsilon)^2\frac{2}{(1+\eta')^2 - 1}} \right) 
        \leq \exp \left( \frac{-\epsilon^2(2+\eta')}{4\eta'(1-\epsilon)^2} \right)
    \end{align}
where in the second to last inequality we used that $\eta' \leq \frac{1}{2m+2}$ to bound $(1+\eta')^{2m+2} < 3$  for every $m$. Substituting the r.h.s of \eqref{eq:hoeffding bound expectation} with $\delta$ and simplifying the term we get that if $\eta' \leq \frac{\epsilon^2}{\log\left(\frac{1}{\delta}\right)}$ then w.p $> 1-\delta$, for every $i=1,\dots,m$ (note that $\tilde{X}_0 = X_0$): 
\[
\min\left\{X_i, 1-\frac{\epsilon}{2}\right\} \leq X_0 - \frac{\epsilon}{2} \leq 1-\epsilon + \frac{\epsilon}{2} = 1 - \frac{\epsilon}{2}.
\]
In particular, the above shows that w.p $> 1-\delta$ for every $i=1,\dots,m$: $X_i = \norm{\bw_i -\bv}^2 \leq 1-\frac{\epsilon}{2}$.
\end{proof}

Next we show that taking a single epoch of $m= O(1/\eta)$ iterations w.h.p will decrease the distance between $\bw$ and $\bv$ by a constant that does not depend on the epoch length or the step size.

\begin{lemma}\label{lem:sgd epoch}
    Let $\delta >0$, take $\eta \leq \frac{\lambda\epsilon_1^2\epsilon_2^2 c_3^2}{60c_1^3c_2^6\log\left(\frac{2}{\delta}\right)}$ where $c_3 = \left( \frac{1}{2}\right)^\frac{\lambda}{20c_1c_2^2} - \left( \frac{1}{2}\right)^\frac{\lambda}{18c_1 c_2^2}$, and $m = \frac{1}{9\eta c_1c_2^2}$. Assume $\epsilon_2 \leq \norm{\bw_0 -\bv}^2 \leq 1-\epsilon_1$. Then w.p $1-\delta$ we have that $\norm{\bw_m -\bv}^2 \leq \left( \frac{1}{2}\right)^\frac{\lambda}{20c_1c_2^2} \norm{\bw_0 -\bv}^2$. 
\end{lemma}

\begin{proof}
    Denote $\tilde{\bw}_i = \mathbb{E}[\bw_i]$ where the expectation is over $\bx_1,\dots,\bx_i$, and let $Z_i = \norm{\bw_i - \tilde{\bw}_i}^2$, then we have that:
    \begin{align}\label{eq:bound martingale Z}
        \left|Z_i - Z_{i-1} \right| &= \left| \left\|\bw_i -  \tilde{\bw}_i\right\|^2 - \left\|\bw_{i-1} -  \tilde{\bw}_{i-1}\right\|^2 \right| \nonumber\\
        & = \left| \left\|\bw_{i-1} - \eta \bg_{i-1} -  \tilde{\bw}_{i-1} +\eta \nabla F(\tilde{\bw}_{i-1}) \right\|^2 - \left\|\bw_{i-1} -  \tilde{\bw}_{i-1}\right\|^2 \right| \nonumber\\
        & \leq 2\eta\left|\inner{\nabla F(\tilde{\bw}_{i-1}) - \bg_{i-1}, \bw_{i-1} - \tilde{\bw}_{i-1}}\right| + \eta^2 \norm{F(\tilde{\bw}_{i-1}) - \bg_{i-1}}^2 \nonumber\\ 
        & \leq 2\eta \norm{F(\tilde{\bw}_{i-1}) - \bg_{i-1}}\cdot\norm{\bw_{i-1} - \tilde{\bw}_{i-1}} + \eta^2 \norm{F(\tilde{\bw}_{i-1}) - \bg_{i-1}}^2 \nonumber\\
        & \leq 2\eta\left( \norm{\nabla F(\tilde{\bw}_{i-1})} + \norm{\bg_{i-1}}\right)\cdot (\norm{\bw_{i-1}} + \norm{\tilde{\bw}_{i-1}}) + \eta\left(\norm{\nabla F(\tilde{\bw}_{i-1})}^2 + \norm{\bg_{i-1}}^2\right)
    \end{align}
    As in the proof of the previous lemma we can bound:
    \[
    \norm{\bg_i}^2 \leq c_1 c_2^2 \norm{\bw_i - \bv}^2 \leq c_1^2 c_2^4
    \]
    where we used our assumption that $\norm{\bw_i - \bv}^2 \leq 1$. In the same manner we can bound $\norm{\nabla F(\tilde{\bw}_{i})} \leq c_1^2c_2^4$. Again using our assumption we have that:
    \[
    \|\bw_i\|\leq \|\bv\| + \norm{\bw_i -\bv} \leq 1+1-\epsilon \leq 2
    \]
    and in the same manner $\norm{\tilde{\bw}_i}\leq 2$. In total we can bound \eqref{eq:bound martingale Z} by:
    \[
    |Z_i - Z_{i-1}| \leq 16\eta c_1^2c_2^4
    \]
    Set $c_3 = \left( \frac{1}{2}\right)^\frac{\lambda}{20c_1c_2^2} - \left( \frac{1}{2}\right)^\frac{\lambda}{18c_1 c_2^2}$, we now us Azuma's inequality and $Z_0 =0 $ to get that:
    \begin{align*}
        \mathcal{P}\left(Z_m \geq \epsilon_2 c_3 \right) \leq \exp\left( \frac{-\epsilon_2^2 c_3^2}{256m\eta^2c_1^4c_2^8}\right)
    \end{align*}
    Substituting the r.h.s with $\frac{\delta}{2}$ we have that for :
    \begin{equation}\label{eq:m small}
        m \leq \frac{\epsilon_2^2 c_3^2}{512c_1^4c_2^8\eta^2\log\left(\frac{2}{\delta}\right)}
    \end{equation} 
    then w.p $>1-\frac{\delta}{2}$: $\norm{\bw_m - \tilde{\bw}_m}^2 \leq \epsilon_2 c_3$. 
    
    Take $m = \frac{1}{9\eta c_1c_2^2}$, by taking $\eta \leq \frac{\lambda\epsilon_1^2\epsilon_2^2 c_3^2}{60c_1^3c_2^6\log\left(\frac{2}{\delta}\right)}$ we have that \eqref{eq:m small} is satisfied and $1-\eta\lambda + \eta^2c \leq 1-\frac{\eta \lambda}{2}$. Finally, using \lemref{lem:w.h.p m iterations w i close to v} with $\frac{\delta}{2}$ and using a union bound, we get that after $m$ iterations w.p $>1-\delta$:
    \begin{align*}
    \norm{\bw_m - \bv}^2 &\leq \norm{\tilde{\bw}_m - \bv}^2 + \norm{\bw_m - \tilde{\bw}_m}^2 \\
    & \leq \left(1-\eta\lambda + \eta^2c\right)^m\norm{\bw_0 -\bv}^2 + \epsilon_2c_3 \\
    & \leq \left(1-\frac{\eta\lambda}{2}\right)^m\norm{\bw_0 -\bv}^2 +  \left(\left( \frac{1}{2}\right)^\frac{\lambda}{20c_1c_2^2} - \left( \frac{1}{2}\right)^\frac{\lambda}{18c_1 c_2^2}\right)\norm{\bw_0 -\bv}^2 \\
    & \leq \left(\left(1-\frac{\eta\lambda}{2}\right)^\frac{2}{\lambda\eta}\right)^\frac{\lambda}{18c_1c_2^2}\norm{\bw_0 -\bv}^2 +  \left(\left( \frac{1}{2}\right)^\frac{\lambda}{20c_1c_2^2} - \left( \frac{1}{2}\right)^\frac{\lambda}{18c_1 c_2^2}\right)\norm{\bw_0 -\bv}^2 \\
    &\leq \left( \frac{1}{2}\right)^\frac{\lambda}{18c_1 c_2^2}\norm{\bw_0 -\bv}^2 +  \left(\left( \frac{1}{2}\right)^\frac{\lambda}{20c_1c_2^2} - \left( \frac{1}{2}\right)^\frac{\lambda}{18c_1 c_2^2}\right)\norm{\bw_0 -\bv}^2 \\
    & \leq \left( \frac{1}{2}\right)^\frac{\lambda}{20c_1c_2^2}\norm{\bw_0 -\bv}^2
    \end{align*}
where in the second to last inequality we used that $(1+x)^\frac{1}{x} \leq \frac{1}{2}$ for $0 \leq x \leq 1$.

\end{proof}

Now we are ready to prove the main theorem, by taking enough epochs with $m$ iterations, and applying union bound:
\begin{proof}[Proof of \thmref{thm:constant probability convergence}(3)]
    We use \lemref{lem:sgd epoch} to get that after $m = \frac{1}{9\eta c_1c_2^2}$ iterations we have w.p $1-\delta$
    \[\norm{\bw_m - \bv}^2 \leq \left( \frac{1}{2}\right)^\frac{\lambda}{20c_1c_2^2} \norm{\bw_0 -\bv}^2.
    \]
    Using the above iteratively for $t$ epochs and applying union bound, we have that after $T = t\cdot m$ iterations w.p $1- t\delta$:
    \[\norm{\bw_{t\cdot m} - \bv}^2 \leq \left( \frac{1}{2}\right)^\frac{t\lambda}{20c_1c_2^2} \norm{\bw_0 -\bv}^2 \leq \left( \frac{1}{2}\right)^\frac{t\lambda}{20c_1c_2^2}.
    \]
    Setting $t = \left\lceil\frac{20c_1c_2^2\log\left(\frac{1}{\epsilon_2}\right)}{\lambda}\right\rceil$ we have w.p $>1 -  \left\lceil\frac{20c_1c_2^2\log\left(\frac{1}{\epsilon_2}\right)}{\lambda}\right\rceil\delta$, after $T = t\cdot m = \frac{2\log\left(\frac{1}{\epsilon_2}\right)}{\lambda\eta} $ iterations we have:
    \[
    \norm{\bw_{T} - \bv}^2 \leq\left( \frac{1}{2}\right)^\frac{t\lambda}{20c_1c_2^2} \leq \epsilon_2
    \]
\end{proof}

\section{Proofs from \secref{sec:spherically symmetric}}\label{appen:proofs from spherically}

In the proofs of this section, we follow the convention that for the ReLU function $\sigma(\cdot)$, it holds that $\sigma'(z)=\mathbf{1}(z\geq 0)$ (and in particular, that $\sigma'(0)=1$). However, the same proofs will hold assuming any other value of $\sigma'(0)$ in $[0,1]$.

\begin{proof}[Proof of \lemref{lem:angle decrease}]
    Using the chain rule and the lemma assumption that $\norm{\bw(t)}>0$ (hence the angle expression is well-defined), we have
    \begin{align*}
    \frac{\partial}{\partial t}\theta(\bw(t),\bv)~&=~
    \frac{\partial}{\partial t}\arccos\left(\frac{\bw(t)^\top\bar{\bv}}{\norm{\bw(t)}}\right)\\
    &=~
    -\frac{1}{\sqrt{1-\left(\frac{\bw(t)^\top\bar{\bv}}{\norm{\bw(t)}}\right)^2}}
    \cdot \left(\frac{\norm{\bw(t)}\bar{\bv}-(\bw(t)^\top\bar{\bv})\frac{\bw(t)}{\norm{\bw(t)}}}
    {\norm{\bw(t)}^2}\right)^\top \left(-\nabla F(\bw(t))\right)\\
    &=~
    \frac{1}{\sqrt{1-\left(\bar{\bw}(t)^\top\bar{\bv}\right)^2}}
    \cdot \left(\frac{\bar{\bv}-(\bar{\bw}(t)^\top\bar{\bv})\bar{\bw}(t)}
    {\norm{\bw(t)}}\right)^\top \nabla F(\bw(t))~.
    \end{align*}
    Thus, it is enough to show that:
    \begin{equation*}
        \left(\bv - \frac{(\bar{\bw}(t)^\top \bv)}{\|\bw(t)\|}\bw(t) \right)^\top \nabla F(\bw(t)) \leq 0.
    \end{equation*}
    We fix $\bw = \bw(t)$, and denote $a= \frac{\bar{\bw}^\top \bv}{\|\bw\|}$. Plugging in the definition of $\nabla F(\bw)$, we want to show that
    \begin{equation*}
        \mathbb{E}_{\bx}\left[\left(\sigma(\bw ^\top \bx) - \sigma(\bv^\top \bx) \right)\cdot \sigma'(\bw^\top \bx)\cdot (\bv^\top \bx -a \bw^\top \bx) \right] \leq 0~.
    \end{equation*}
    Using the assumption that $\sigma$ is ReLU, the above can be rewritten as
    \begin{equation}\label{eq:symmetric distribution what we need to show}
        \mathbb{E}_{\bx}\left[\left(\sigma(\bw ^\top \bx) - \sigma(\bv^\top \bx) \right)\cdot (\bv^\top \bx -a \bw^\top \bx)\cdot\mathbbm{1}(\bw^\top \bx \geq 0) \right] \leq 0~.
    \end{equation}
    We now note that the expression above depends only on inner products of $\bx$ with $\bw,\bv$, so we can rewrite the inequality as
    \[
    \mathbb{E}_{\by\sim\mathcal{D}_{\bw,\bv}}
    \left[\left(\sigma(\hat{\bw} ^\top \by) - \sigma(\hat{\bv}^\top \by) \right)\cdot (\hat{\bv}^\top \by -a \hat{\bw}^\top \by)\cdot \mathbbm{1}(\hat{\bw}^\top \by \geq 0) \right] \leq 0~,
    \]
    where $\mathcal{D}_{\bw,\bv}$ is the marginal distribution of $\bx$ on the 2-dimensional subspace $\text{span}\{\bw,\bv\}$, and $\hat{\bw},\hat{\bv}\in \reals^2$ are the representations of $\bw,\bv$ in that subspace. Moreover, by the spherical symmetry of the distribution, the expression above is invariant to rotating the coordinate frame, so we can assume without loss of generality that $\hat{\bw}=\norm{\bw}\begin{pmatrix} 1 \\ 0\end{pmatrix}$, in which case the above reduces to
    \begin{equation*}
        \mathbb{E}_{\by\sim\mathcal{D}_{\bw,\bv}} \left[\left(\|\bw\|\begin{pmatrix} 1 \\ 0\end{pmatrix}^\top \by - \sigma(\hat{\bv}^\top \by) \right)\cdot \left(\hat{\bv}^\top \by -\inner{\bar{\bw},\bv} \begin{pmatrix} 1 \\ 0\end{pmatrix}^\top \by\right)\cdot \mathbbm{1}( y_1>0) \right] \leq 0~.
    \end{equation*}

    Denote $g(\by) = \left(\|\bw\|\begin{pmatrix} 1 \\ 0\end{pmatrix}^\top \by - \sigma(\hat{\bv}^\top \by) \right)\cdot \left(\hat{\bv}^\top \by -\inner{\bar{\bw},\bv} \begin{pmatrix} 1 \\ 0\end{pmatrix}^\top \by\right)$, so that the inequality above is
    \begin{equation}
    \mathbb{E}_{\by\sim\mathcal{D}_{\bw,\bv}}[~g(\by)\cdot \mathbbm{1}(y_1>0)]\leq 0~.
    \label{eq:2d symmetric what we want to show}
    \end{equation}
    The function $g(\by)$ can be simplified as:
    \begin{equation*}
        g(\by) = (\|\bw\|y_1 - \sigma(y_1\hat{v}_1 + y_2\hat{v}_2))\cdot (y_1\hat{v}_1 + y_2\hat{v}_2 - \hat{v}_1y_1)= (\|\bw\|y_1 - \sigma(y_1\hat{v}_1 + y_2\hat{v}_2))\cdot y_2\hat{v}_2~,
    \end{equation*}
     where we used the fact that $\inner{\bar{\bw},\bv} = \inner{\frac{1}{\norm{\bw}}\hat{\bw},\hat{\bv}} = v_1$.
     
    We now perform a case analysis to justify \eqref{eq:2d symmetric what we want to show}, depending on the value of $a$ (which by definition, equals $\frac{\bar{\bw}^\top\bv}{\norm{\bw}}=\frac{\bw^\top\bv}{\norm{\bw}^2}=\frac{\hat{\bw}^\top\hat{\bv}}{\norm{\bw}^2}=\frac{\hat{v}_1}{\norm{\bw}}$). In all the cases we assume $y_1 > 0$, otherwise the expression in the expectation is zero.
    \begin{itemize}
        \item $0\leq a \leq 1$: In this case $\hat{v}_1 \geq 0$, and also $\inner{\bar{\bw},\bv} \leq \|\bw\| $. Assume w.l.o.g that $\hat{v}_2 \geq 0$ (the other case is similar), and for $\by = \begin{pmatrix} y_1 \\ y_2 \end{pmatrix}$ denote $\tilde{\by} = \begin{pmatrix} y_1 \\ -y_2 \end{pmatrix}$. If $y_2 <0$ then $g(\by) \leq 0$, on the other hand if $y_2 > 0$ then we can rewrite:
    \begin{equation*}
        g(\by) = y_2\hat{v}_2\cdot(y_1(\|\bw\|-\hat{v}_1) - y_2\hat{v}_2) = y_2\hat{v}_2\cdot(y_1(\|\bw\|-\inner{\bar{\bw},\bv}) - y_2\hat{v}_2),
    \end{equation*}
    where we have two cases:
    \begin{enumerate}
        \item if $y_1(\|\bw\|-\inner{\bar{\bw},\bv}) > y_2\hat{v}_2$ then $|g(\tilde{\by})| \geq g(\by)$ and also $g(\tilde{\by}) \leq 0$
        \item If $y_1(\|\bw\|-\inner{\bar{\bw},\bv}) \leq y_2\hat{v}_2$ then $g(\by) \leq 0$.
    \end{enumerate}
    
    We showed that for every $\by\in \mathbb{R}^2$ either $g(\by) \leq 0$ or there is a unique $\tilde{\by}\in\mathbb{R}^2$ with the same norm as $\by$ such that $g(\tilde{\by})\leq 0$ and $|g(\tilde{\by})| \geq g(\by)$. Since $\mathcal{D}$ has a spherical symmetric distribution this shows that \eqref{eq:2d symmetric what we want to show} holds for these values of $a$.
    \item $a\leq 0$: In this case $\hat{v}_1\leq 0$, we also assume w.l.o.g that $\hat{v}_2 \geq 0$ (the other case is similar). Here for every $\by$ with $y_2 \leq 0$ we have that:
    \begin{equation*}
        g(\by) = (\|\bw\|y_1 - \sigma(y_1\hat{v}_1 + y_2\hat{v}_2))\cdot y_2\hat{v}_2 = \|\bw\|y_1\cdot y_2\hat{v}_2 \leq 0,
    \end{equation*}
    because $y_1 \geq 0$. On the other hand, if $y_2 \geq 0$ we have two cases:
    \begin{enumerate}
        \item If also $\hat{v}_1y_1+\hat{v}_2y_2 \leq 0$ then $g(\by) = \|\bw\|y_1\cdot y_2\hat{v}_2 \geq 0$, and then $g(\tilde{\by}) = -g(\by)$.
        \item If $\hat{v}_1y_1+\hat{v}_2y_2 \geq 0$ then $g(\by) = (\|\bw\|y_1 - \hat{v}_1y_1-\hat{v}_2y_2)\cdot y_2\hat{v}_2$. If $g(\by) \geq 0$, then $g(\tilde{\by}) \leq 0$ and also $|g(\tilde{\by})| \geq g(\by)$.
    \end{enumerate} 
    Hence we proved that for every $\by$ with $y_1>0$ either $g(\by) \leq 0 $ or there is $\tilde{\by}$ with $|g(\tilde{\by})| \geq g(\by)$ and $g(\tilde{\by}) \leq 0$. Since $\mathcal{D}$ has a spherical symmetric distribution this shows that \eqref{eq:2d symmetric what we want to show} holds for these values of $a$.
    \item $a \geq 1$:
    In this case $\hat{v}_1 \geq 0$ and $\inner{\hat{\bw},\bv} \geq \norm{\bw}$. Assume w.l.o.g that $\hat{v}_2 \geq 0$ (the other case is similar). If $y_2 >0$ then
    \[
    g(\by) = (y_1(\norm{\bw} - \hat{v}_1) - y_2\hat{v}_2)\cdot y_2\hat{v}_2 \leq 0.
    \]
    If $y_2 < 0 $ then we have two case:
    \begin{enumerate}
        \item $y_1\hat{v}_1 + y_2\hat{v}_2 \leq 0$, then $g(\by) = \norm{\bw}y_1\cdot y_2\hat{v}_2 <0$
        \item $y_1\hat{v}_2 + y_2\hat{v}_2 >0$, in which case if $g(\by) >0$ then $g(\tilde{\by}) < 0$ and $g(\tilde{\by})  \geq g(\by)$.
    \end{enumerate}
    Hence for every $\by$ with $y_1>0$ either $g(\by) <0$ or there is $\tilde{\by}$ with $g(\tilde{\by}) < 0$ and $g(\tilde{\by})  \geq g(\by)$. This shows that \eqref{eq:2d symmetric what we want to show} holds for these values of $a$.
    \end{itemize}
\end{proof}

    \begin{proof}[Proof of \lemref{lem:norm gets larger}]
   By our assumption $\bw(t)\neq 0$, hence the gradient of the objective is well-defined and we have that
    \begin{equation}\label{eq:gradient of the norm}
        \frac{\partial}{\partial t}\|\bw(t)\|^2 = -\bw(t)^\top\nabla F(\bw(t)) = \mathbb{E}_\bx\left[\left(\sigma(\bv^\top \bx) - \sigma(\bw(t)^\top \bx)\right)\sigma'(\bw(t)^\top \bx)\bw(t)^\top \bx \right]~.
    \end{equation}
     Fix $\bw = \bw(t)$. Using the assumption that $\sigma$ is the ReLU function we can rewrite \eqref{eq:gradient of the norm} as:
    \begin{equation}\label{eq:gradient of the norm relu}
        \mathbb{E}_\bx \left[\left(\sigma(\bv^\top \bx) - {\bw}^\top\bx \right)\cdot {\bw}^\top \bx \cdot \mathbbm{1}({\bw}^\top\bx \geq 0) \right]~.
    \end{equation}
    Since the function inside the expectation in \eqref{eq:gradient of the norm relu} depends only on the inner product of  $\bx$ with $\bw$ and $\bv$, we can consider the marginal distribution $\mathcal{D}_{\bw,\bv}$ on the 2-dimensional subspace $\text{span}\{\bw,\bv\}$, we also denote $\hat{\bw},\hat{\bv}\in \reals^2$ as the representations of $\bw,\bv$ on this 2-dimensional subspace. We can now rewrite \eqref{eq:gradient of the norm relu} as:
    \begin{equation}\label{eq:2d gradient norm gets smaller}
        \mathbb{E}_{\by\sim\mathcal{D}_{\bw,\bv}} \left[\left(\sigma(\hat{\bv}^\top \by) - \hat{{\bw}}^\top\by \right)\cdot \hat{{\bw}}^\top \by \cdot \mathbbm{1}(\hat{{\bw}}^\top\by \geq 0) \right]~.
    \end{equation}
    Note that the function inside the expectation in \eqref{eq:2d gradient norm gets smaller} is homogeneous with respect to the norm of $\by$. Also, by our assumption $\mathcal{D}$ is a spherically symmetric distribution, hence also $\mathcal{D}_{\bw,\bv}$ is spherically symmetric. Thus, in order to prove that \eqref{eq:2d gradient norm gets smaller} is non-negative, it is enough to consider the conditional distribution $\mathcal{D}_{\bw,\by,1}$ of $\by$ on the set $\{\by: \|\by\|=1\}$. Since $\mathcal{D}_{\bw,\bv,1}$ (as a distribution on $\reals^2$) is still spherically symmetric, it is invariant to a rotation of the coordinate system, so we can assume w.l.o.g that $\hat{\bw} = \|\bw\|\begin{pmatrix} 1 \\ 0 \end{pmatrix}$. Overall, in order to prove that \eqref{eq:2d gradient norm gets smaller} is non-negative it is enough to show that:
     \begin{align}\label{eq:2d gradient norm gets smaller simplified}
    & \mathbb{E}_{\by\sim\mathcal{D}_{\bw,\bv,1}}\left[(\sigma(\hat{v}_1y_1 + \hat{v}_2y_2) - \|\bw\|y_1)\mathbbm{1}(y_1 \geq 0)\cdot y_1\norm{\bw} \right] \geq 0~.
    \end{align}
    
    Since $\mathcal{D}$ is spherically symmetrical and the function inside \eqref{eq:2d gradient norm gets smaller simplified}, the marginal distribution $\mathcal{D}_{\bw,\bv,1}$ is actually a uniform distribution on $\{\by\in\mathbb{R}^2: \|\by\|=1\}$. Thus, in order to show that \eqref{eq:2d gradient norm gets smaller simplified} is non-negative, we can divide it by $\norm{\bw}$ (which is positive), and show that the following integral is non-negative:
    \begin{align*}
        &\int_0^1 \left(\sigma\left(v_1y_1 + v_2\sqrt{1-y_1^2}\right) - \|\bw\|y_1\right)y_1 + \left(\sigma\left(v_1y_1 - v_2\sqrt{1-y_1^2}\right) - \|\bw\|y_1\right)y_1 dy_1 \nonumber\\
        = & \int_0^1 y_1\left(\sigma\left(v_1y_1 + v_2\sqrt{1-y_1^2}\right) + \sigma\left(v_1y_1 - v_2\sqrt{1-y_1^2}\right)\right) - 2\|\bw\|y_1^2dy_1,
    \end{align*}

    where we wrote $y_2=\pm\sqrt{1-y_1^2}$ since $\|\by\|=1$. We can assume w.l.o.g that $v_2 \geq 0$ (the other direction is similar) and write $v_2 = \sqrt{1-v_1^2}$, and thus it is enough to prove that:
    \begin{equation}\label{eq:int of norm}
        \int_0^1 y_1\sigma\left(v_1y_1 + \sqrt{(1-y_1^2)(1-v_1^2)}\right) - 2\|\bw\|y_1^2dy_1 \geq 0~.
    \end{equation}
    Denote $\theta=\theta(\bw,\bv)$, since $\inner{\bar{\bw},\bv}=\inner{\hat{\bar{\bw}},\hat{\bv}}=v_1$ then $v_1=\cos(\theta)$ and $\sqrt{1-v_1^2}=\sin(\theta)$. Now we split into cases for the different values of $v_1$:
    \begin{itemize}
        \item $v_1 \geq 0$:
    In this case, if $0\leq y_1 \leq 1$ then $v_1y_1 + \sqrt{(1-y_1^2)(1-v_1^2)} \geq 0$, hence the integral in \eqref{eq:int of norm} can be calculated as:
    \begin{align}
        \int_0^1 y_1\left(v_1y_1 + \sqrt{(1-y_1^2)(1-v_1^2)}\right) - 2\|\bw\|y_1^2dy_1 = \frac{v_1}{3} + \frac{\sqrt{1-v_1^2}}{3} - \frac{2\|\bw\|}{3}.
    \end{align}
    Thus, the above term is non-negative if:
    \begin{equation*}
        \|\bw\| \leq \frac{v_1 + \sqrt{1-v_1^2}}{2} = \frac{\sin(\theta) + \cos(\theta)}{2}.
    \end{equation*}
    \item $v_1 \leq 0$:
    In this case, if $0\leq y_1 \leq \sqrt{1-v_1^2}$ then $v_1y_1 + \sqrt{(1-y_1^2)(1-v_1^2)} \geq 0$, and if $\sqrt{1-v_1^2} < y_1 \leq 1$ then $v_1y_1 + \sqrt{(1-y_1^2)(1-v_1^2)} \leq 0$. Thus, the integral in \eqref{eq:int of norm} can be calculated as:
    \begin{align*}
        &\int_0^{\sqrt{1-v_1^2}}  y_1\left(v_1y_1 + \sqrt{(1-y_1^2)(1-v_1^2)}\right) - 2\|\bw\|y_1^2dy_1 - \int_{\sqrt{1-v_1^2}}^1 2\|\bw\|y_1^2 dy_1 \\
        = & -\frac{2\|\bw\|}{3}  + \frac{v_1^3\sqrt{1-v_1^2}}{3} + \frac{v_1\left(\sqrt{1-v_1^2}\right)^3}{3} + \frac{\sqrt{1-v_1^2}}{3} \\
        = & -\frac{2\|\bw\|}{3} + \frac{\sqrt{1-v_1^2}(1+v_1)}{3}.
    \end{align*}
    Thus, the above term is non-negative if:
    \begin{equation*}
        \|\bw\| \leq \frac{\sqrt{1-v_1^2}(1+v_1)}{2} = \frac{\sin(\theta)(1+\cos(\theta))}{2}
    \end{equation*}
    \end{itemize}
\end{proof}

\begin{proof}[Proof of \thmref{thm: converge w.h.p}]
    Assume we initialized with $\theta(\bw(0),\bv) \leq \pi - \epsilon$ and $0<\norm{\bw(0)} \leq 2$. First we will show that $\bw(t)\neq 0$ for all $t>0$. Assume on the way of contradiction that for some $t>0$ we have $\bw(t)=0$, and let $t_1$ be the first time for which it happens. For $t_0=0$ we know that $\bw(t_0)\neq 0$, and also that for all $t\in [t_0,t_1]$, $\bw(t)\neq 0$ and the gradient of the objective is well defined. Hence by \lemref{lem:angle decrease} we know that $\theta(\bw(t),\bv)\leq \pi - \epsilon$ for all $t\in[t_0,t_1]$, because the angle can only decrease unless $\bw(t)=0$. But, by \lemref{lem:norm gets larger} we know that if $\norm{\bw(t)} \leq \max\left\{\frac{\sin(\epsilon) - \cos(\epsilon)}{2}, \frac{\sin(\epsilon)(1-\cos(\epsilon))}{2} \right\}$ then $\frac{\partial}{\partial t}\norm{\bw(t)} \geq 0$. In particular, for $\epsilon \in (0,\pi]$ and for all $t_0\leq t< t_1$, we have that $\norm{\bw(t)}$ is bounded below by $\max\left\{\frac{\sin(\epsilon) - \cos(\epsilon)}{2}, \frac{\sin(\epsilon)(1-\cos(\epsilon))}{2} \right\} > 0$, a contradiction to $\bw(t_1) = 0$. This shows that for all $t>0$ we have that $\bw(t)\neq 0$, hence by \lemref{lem:angle decrease} we know that for every $t>0$ we will have $\theta(\bw(t),\bv) \leq \pi - \epsilon$.
    
    Now we can use \thmref{thm:innerprod} (where $\gamma=1$ because of Assumption \ref{assump:spherically}(3)) to get:
    \begin{equation*}
        \inner{\nabla F(\bw(t)), \bw(t) - \bv} \geq \frac{\alpha^4\beta}{8\sqrt{2}}\sin\left(\frac{\epsilon}{8}\right)^3\|\bw(t)-\bv\|^2.
    \end{equation*}
     Set $\lambda=\frac{\alpha^4\beta}{8\sqrt{2}}\sin\left(\frac{\epsilon}{8}\right)^3$, as explained above for all $t>0$, $\nabla F(\bw(t))$ is continuous since $\bw(t)\neq 0$ and we have that: 
    \begin{align*}
        \frac{\partial}{\partial t}\|\bw(t)-\bv\|^2 &= 2\inner{\bw(t)-\bv,\frac{\partial}{\partial t} \bw(t)} = -2\inner{\bw(t)-\bv,\nabla F(\bw(t))} \leq -\lambda\|\bw(t)-\bv\|^2,
    \end{align*}
     Using Gr\"{o}nwall's inequality, this proves that for every $t>0$ we get:
    \begin{equation*}
        \|\bw(t) - \bv\|^2 \leq \|\bw(0)-\bv\|^2\exp(-\lambda t).
    \end{equation*}
\end{proof}

\subsection{Standard Gaussian Distribution}\label{subsec:standard gaussian proofs}
In this subsection we assume that $\mathcal{D} = \mathcal{N}(0,I)$, and that $\sigma$ is the ReLU function.

\begin{lemma}\label{lem:angleincrease}
    If $\bw(t) \neq 0$, then $\frac{\partial}{\partial t}\theta(\bw(t),\bv)\leq 0$
\end{lemma}

\begin{proof}
Similar to the proof of \lemref{lem:angle decrease}, it is enough to prove that
\begin{equation}\label{eq:toshownonpos}
\left(\bar{\bv}-(\bar{\bw}(t)^\top\bar{\bv})\bar{\bw}(t)\right)^\top
\nabla F(\bw(t))~\leq~ 0~,
\end{equation}
where we used that $\norm{\bw(t)} > 0$ hence the angle expression is differentiable. In the standard Gaussian case, $\nabla F(\bw(t))$ has a closed-form expression (see \cite{brutzkus2017globally}, \cite{safran2017spurious}), namely
\begin{equation}\label{eq:gradgauss}
\nabla F(\bw) = \frac{1}{2}\bw-\frac{1}{2\pi}\left(\norm{\bv}\sin(\theta(\bw,\bv))\bar{\bw}+(\pi-\theta(\bw,\bv)\bv)\right)~.
\end{equation}
Multiplying this by $\left(\bar{\bv}-(\bar{\bw}(t)^\top\bar{\bv})\bar{\bw}(t)\right)$, and noting that this vector is orthogonal to $\bw(t)$ (as it is simply the component of $\bar{\bv}$ orthogonal to $\bar{\bw}(t)$, we get that
\begin{align*}
\left(\bar{\bv}-(\bar{\bw}(t)^\top\bar{\bv})\bar{\bw}(t)\right)^\top
\nabla F(\bw(t))
~&=~ 
\left(\bar{\bv}-(\bar{\bw}(t)^\top\bar{\bv})\bar{\bw}(t)\right)^\top
\left(-\frac{\pi-\theta(\bw(t),\bv)}{2\pi}\bv\right)\\
&=~
-\frac{\pi-\theta(\bw(t),\bv)}{2\pi}\left(\bar{\bv}^\top \bv-(\bar{\bw}(t)^\top\bar{\bv})
(\bar{\bw}(t)^\top\bv)\right)\\
&=~
-\frac{\pi-\theta(\bw(t),\bv)}{2\pi}\left(\norm{\bv}-\norm{\bv}(\bar{\bw}(t)^\top\bar{\bv})^2
\right)\\
&=~-\frac{\pi-\theta(\bw(t),\bv)}{2\pi}\left(1-(\bar{\bw}(t)^\top\bar{\bv})^2
\right)\norm{\bv}~.
\end{align*}
Since $\theta(\bw(t),\bv)\in [-\pi,\pi]$ and $\bar{\bw}(t)^\top\bar{\bv}\in [-1,1]$, it follows that this expression is non-negative, establishing \eqref{eq:toshownonpos} and hence the lemma.

\end{proof}

\begin{lemma}\label{lem:normincrease}
    Let $\theta(\bw(t),\bv)= \pi - \alpha$ and assume that $\bw(t)\neq 0$. If $\norm{\bw(t)}\leq \frac{\norm{\bv}}{\pi^4}\alpha^3$, then $\frac{\partial}{\partial t}\norm{\bw(t)}^2\geq 0$
\end{lemma}

\begin{proof}

Using the closed-form expression for $\nabla F(\bw)$ (see \eqref{eq:gradgauss}), we have
\begin{align*}
\frac{\partial}{\partial t}\norm{\bw(t)}^2~&=~
\bw(t)^\top\frac{\partial}{\partial t}\bw(t)~=~
-\bw(t)^\top\nabla F(\bw(t))\\
&=~-\frac{\norm{\bw(t)}^2}{2}+\frac{1}{2\pi}
\left(\norm{\bv}\norm{\bw(t)}\sin(\theta(\bw(t),\bv))+(\pi-\theta(\bw(t),\bv)\bw(t)^\top\bv)\right)\\
&=~\frac{\norm{\bw(t)}\norm{\bv}}{2\pi}\left(\sin(\theta(\bw(t),\bv))+(\pi-\theta(\bw(t),\bv))\bar{\bw}(t)^\top\bar{\bv}-\frac{\pi\norm{\bw(t)}}{\norm{\bv}}\right)\\
&=~\frac{\norm{\bw(t)}\norm{\bv}}{2}\left(\sin(\theta(\bw(t),\bv))+(\pi-\theta(\bw(t),\bv))\cos(\theta(\bw(t),\bv))-\frac{\pi\norm{\bw(t)}}{\norm{\bv}}\right)\\
\end{align*}
The expression $\sin(\theta)+(\pi-\theta)\cos(\theta)$ can be easily verified to be strictly monotonically decreasing in $\theta\in (0,\pi)$, and equal $0$ at $\theta=\pi$. Therefore, if $\theta\leq \pi-\alpha$, then the expression above can be lower bounded by
\begin{equation}
\frac{\norm{\bw(t)}\norm{\bv}}{2}\left(\sin(\pi-\alpha)+\alpha\cos(\pi-\alpha)-\frac{\pi\norm{\bw(t)}}{\norm{\bv}}\right)~=~
\frac{\norm{\bw(t)}\norm{\bv}}{2}\left(\sin(\alpha)-\alpha\cos(\alpha)-\frac{\pi\norm{\bw(t)}}{\norm{\bv}}\right)~.\label{eq:normlowbound}
\end{equation}
To slightly simplify this expression, we will now argue that 
\begin{equation}\label{eq:alpha}
\sin(\alpha)-\alpha\cos(\alpha)\geq \left(\frac{\alpha}{\pi}\right)^3~~~\forall \alpha \in [0,\pi]~.
\end{equation}
Assuming this inequality holds, we get that \eqref{eq:normlowbound} is at least
\[
\frac{\norm{\bw(t)}\norm{\bv}}{2}\left(\left(\frac{\alpha}{\pi}\right)^3-\frac{\pi\norm{\bw(t)}}{\norm{\bv}}\right)~,
\]
which is non-negative as long as $\norm{\bw(t)}\leq \norm{\bv}\alpha^3/\pi^4$, proving the lemma. It only remains to establish \eqref{eq:alpha}. We consider two cases:
\begin{itemize}
	\item If $\alpha\in [0,\pi/2]$, then by a Taylor expansion of $\sin(\alpha),\cos(\alpha)$ around $0$, we have
	\[
	\sin(\alpha)-\alpha\cos(\alpha)\geq \alpha-\frac{\alpha^3}{3!}-\alpha\left(1-\frac{\alpha^2}{2!}+\frac{\alpha^4}{4!}\right)~=~\alpha^3\left(\frac{1}{2!}-\frac{1}{3!}-\frac{\alpha^2}{4!}\right)
	~\geq~\alpha^3\left(\frac{1}{2!}-\frac{1}{3!}-\frac{(\pi/2)^2}{4!}\right)
	\]
	which is at least $\alpha^3/5$.
	\item If $\alpha \in \left[\frac{\pi}{2},\pi\right]$, it is easily verified via differentiation that $\sin(\alpha)-\alpha\cos(\alpha)\geq \sin(\alpha)$ is monotonically increasing in $\alpha$. Therefore, it can be lower bounded by $\sin(\pi/2)-(\pi/2)\cos(\pi/2)=1\geq \alpha^3/\pi^3$.
\end{itemize}
Combining the two cases, \eqref{eq:alpha} follows.
\end{proof}

\end{document}